\documentclass{amsart}

\usepackage[utf8]{inputenc}
\usepackage[T1]{fontenc}
\usepackage{amsmath,amsthm,amssymb,mathtools}
\usepackage{mathrsfs}
\usepackage{enumitem}
\usepackage{graphicx,booktabs,caption,placeins,xcolor}
\graphicspath{{figures/qk_geometry/}}
\usepackage[margin=1in]{geometry}
\usepackage{xurl}
\usepackage[colorlinks=true,hyperindex,pagebackref=false,citecolor=cyan,linkcolor=cyan]{hyperref}
\mathtoolsset{showonlyrefs}

\newcommand{\defn}[1]{\emph{#1}}
\newcommand{\RR}{\mathbb R}
\newcommand{\EE}{\mathbb E}
\newcommand{\End}{\operatorname{End}}

\newcommand{\id}{\operatorname{id}}
\newcommand{\tr}{\operatorname{tr}}

\theoremstyle{plain}
\newtheorem{theorem}{Theorem}
\newtheorem{lemma}{Lemma}
\newtheorem{proposition}{Proposition}
\newtheorem{corollary}{Corollary}
\theoremstyle{definition}
\newtheorem{definition}{Definition}
\newtheorem{observation}{Observation}
\newtheorem{example}{Example}
\theoremstyle{remark}
\newtheorem{remark}{Remark}

\definecolor{typeIpluscolor}{HTML}{1f77b4}
\definecolor{typeIIcolor}{HTML}{e07a00}
\definecolor{typeIminuscolor}{HTML}{5b5b5b}
\newcommand{\typeIplus}{{\color{typeIpluscolor}$+$}}
\newcommand{\typeII}{{\color{typeIIcolor}$\circ$}}
\newcommand{\typeIminus}{{\color{typeIminuscolor}$-$}}

\makeatletter
\def\subsection{\@startsection{subsection}{2}%
  \z@{.5\linespacing\@plus.7\linespacing}{.3\linespacing}%
  {\normalfont\bfseries}}
\makeatother

\title{On attention heads and bilinear forms}

\author{Andrew O'Desky}
\date{September 17, 2026.}

\begin{document}

\begin{abstract}
We study the symmetric and antisymmetric parts of bilinear forms in the attention heads of trained large language models. We introduce an orthogonally invariant profile map from real bilinear forms to a three-dimensional simplex and observe that profiles of trained bilinear forms accumulate near profiles of rank-one bilinear forms. We prove that the symmetric part of a bilinear form in an attention head is the sum of a hyperbolic form and a zero form for a Zariski-dense subset of query-key matrices.
\end{abstract}

\maketitle

\section{Introduction}

The stunning success of modern language models relies primarily on computational scale
and the discovery of a new class of functions known as \emph{transformers} \cite{vaswani2017attention}.
Transformers are remarkably versatile
at approximating complicated functions
traditionally requiring sophisticated algorithms or human cognition,
yet their versatility is not well-understood.
Researchers from diverse areas
have proposed different models for interpreting them
but it remains unclear why transformers are so effective.

The novel ingredient in transformers is the attention mechanism.
Language to be processed is modeled as
a sequence of vectors in a real vector space,
and the attention mechanism facilitates real-time interactions
between words which encode context
for later stages of the transformer.
These interactions are computed using
a pretrained \emph{bilinear form} encoded in
the attention weights.
While real bilinear forms are well-understood mathematically,
the relationship between
properties of the bilinear forms in attention heads
and aspects of language processing is poorly understood.
This paper is an attempt to better understand that relationship.
We also give a formulation of
transformers equivalent to the usual construction
which we hope will be inviting to mathematicians
who are interested in transformers.
For readers from the machine learning community,
we have attempted to provide enough background so that
the paper is self-contained.

\subsection{Symmetric bilinear forms in attention heads are balanced}

Any bilinear form $B$ can be uniquely decomposed as $B = S + T$
where $S$ is {symmetric} and $T$ is {antisymmetric}.
In recent work \cite{migliarini2025selfhating}
on the \texttt{gpt2} language model,
Migliarini observed that the symmetric part $S$
of one of its head's bilinear forms had 33 negative eigenvalues among
the top 64 eigenvalues when sorted by absolute magnitude,
and hypothesized that the presence of negative eigenvalues
was responsible for the ``self-suppressing'' behavior of this head.
Migliarini's observation raises many questions about
the distributions of these eigenvalues in general,
foremost being whether this distribution is atypical.
We inspected all $9{,}512$ heads in the fourteen models listed in
\S\ref{sec:experimentsQK}.  Their eigenvalues give a striking
observation.

\begin{observation}\label{obs:balance}
Every attention head $A$ has a
bilinear form whose symmetric part $S$ has precisely $2n$ nonzero
eigenvalues, precisely half of which are negative, where $n=n_A$ is the
head dimension.
\end{observation}

In the language of bilinear forms, the symmetric bilinear form in
every attention head $A$ has signature $(n_A,n_A,N_A-2n_A)$
where $N_A$ is the embedding dimension.
This observation is explained by the following theorem
which proves this is a linear-algebraic consequence of
the standard query--key construction for attention heads.
Let $C$ be a real vector space of dimension $N$ and let $\ell$ be a positive integer.
Let $H$ be a real vector space of dimension $n$.
For linear maps $Q, K \colon C \to H$,
let $A\colon C^\ell \to C^\ell$ denote the length $\ell$ attention head on $C$
constructed from $Q$ and $K$ as query and key transformations.

\begin{theorem}[Theorem~\ref{thm:heads-have-balanced-attention}]\label{thm:balancedIntro}
Let $S$ denote the symmetric component of the bilinear form
of $A$.
If the stacked map
\[
    \binom{Q}{K} \colon C \to H\oplus H,
    \qquad
    x \mapsto (Qx, Kx)
\]
is surjective, then the signature of $S$ is $(n,n,N-2n)$.
In particular, if the head dimension is
no more than half the embedding dimension 
then there is a Zariski-dense subset of query and key transformations 
whose attention head has this property.
\end{theorem}

The theorem can also be applied to RoPE and ALiBi attention 
to obtain the same conclusion. 

\subsection{Profiles of bilinear forms in language models}

We next turn to properties which are learned rather than structural.
To measure these we introduce a map on
real bilinear forms valued in the $3$-simplex
which is invariant under positive scaling and orthogonal changes of coordinates
(Definition~\ref{defn:profileMap}):
\[
    \pi \colon M_N(\RR) \smallsetminus \{0\} \longrightarrow \Delta_3,
    \qquad
    L=S+T \longmapsto (a, b, c, d), \qquad a,b,c \geq 0,\,\,\, a + b + c \leq 1.
\]
We call $\pi(L)$ the \defn{profile} of $L$.
The purpose of $\pi$ is to give a qualitative description of bilinear forms.
The fibers of $\pi$ over the faces of $\Delta_3$ can be explicitly described
(Theorem~\ref{thm:simplex-faces}), 
and the profiles of rank-one forms fill out 
a one-parameter family $\Theta \subset \Delta_3$. 

\begin{figure}[!htbp]
\centering
\includegraphics[width=0.9\textwidth]{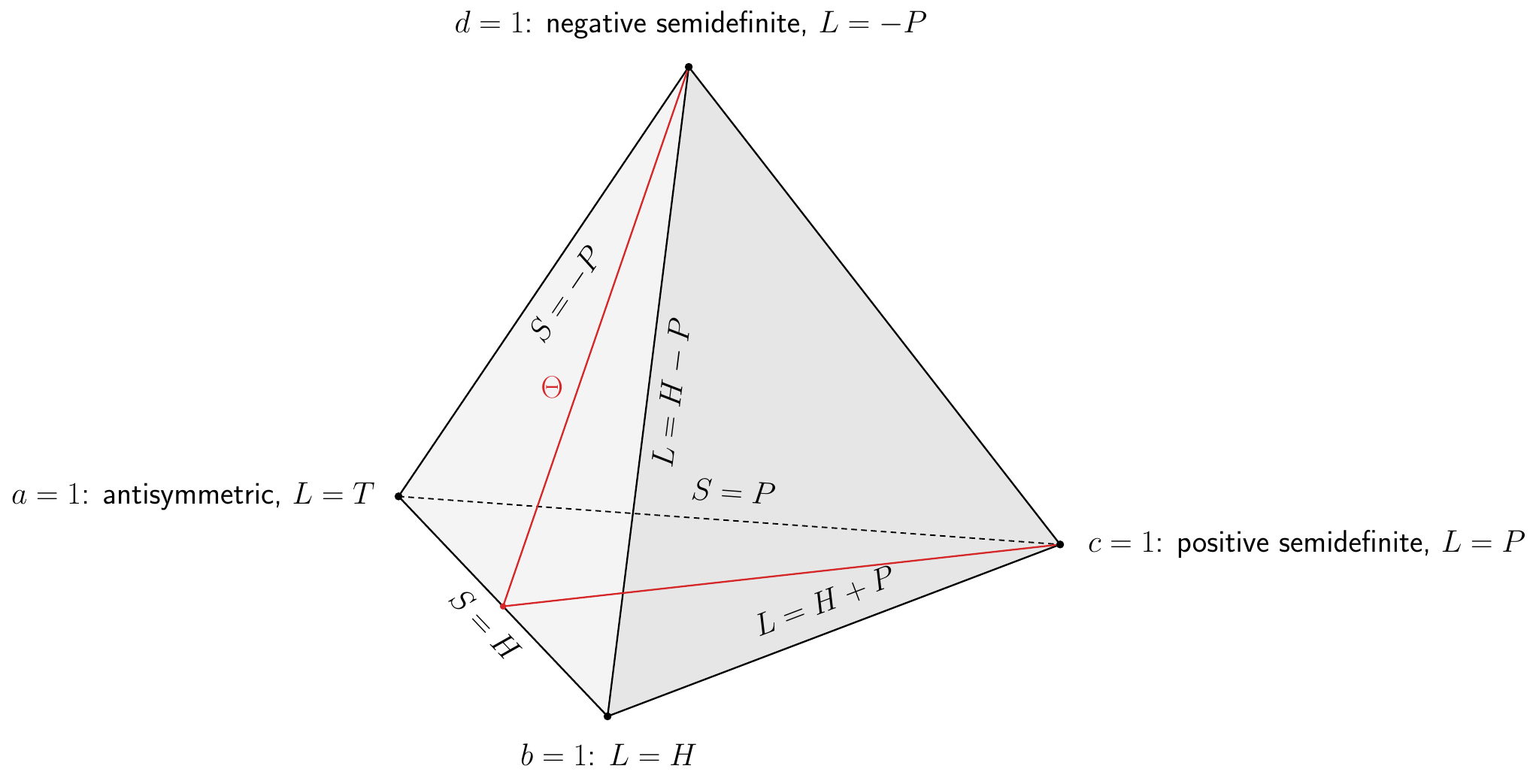}
\caption{The simplex $\Delta_3$ of profiles of nonzero bilinear forms.
Red segments comprise $\Theta = \{a=b,\;cd=0\}$. 
$H$ denotes a symmetric matrix whose spectrum is symmetric
about zero and $P$ a positive semidefinite matrix.
The front-right facet $a=0$ corresponds to symmetric forms,
the bottom facet $d = 0$ to $S = H + P$,
and the left facet to $S = H - P$.
Fibers of $\pi$ over the relative interiors of
the back facet $\{b = 0\}$
and the top-right edge are empty.}
\label{fig:simplex-labels}
\end{figure}

We examined the profiles of bilinear forms
from fourteen language models.
For large rank, the image of the profile map fills up most of the simplex.
Nonetheless, the profiles of trained bilinear forms from language models
tend to accumulate near profiles of rank-one forms.

\begin{observation}\label{obs:head-profiles}
In every one of the fourteen models except \texttt{opt-350m},
the profiles of their bilinear forms accumulate near
the one-parameter family
$\Theta = \{a = b,\; cd = 0\} \subset \Delta_3$
of profiles of rank-one forms.
\end{observation}

\begin{figure}[!htbp]
\centering
\includegraphics[width=\textwidth]{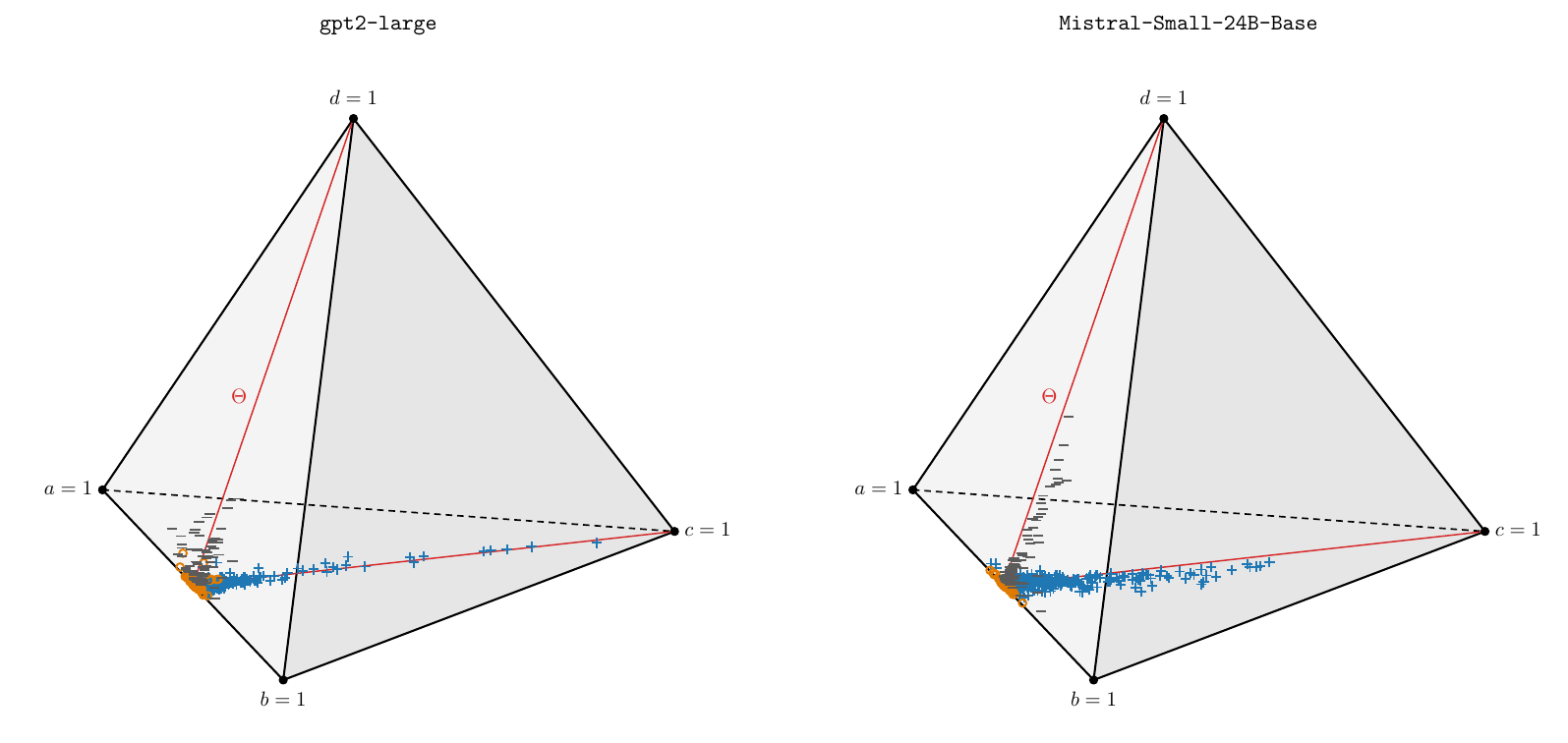}
\caption{Profiles of the bilinear forms
of \texttt{gpt2-large} and of \texttt{Mistral-Small-24B-Base}.
The left model uses QKV attention while the right uses RoPE,
showing that accumulation occurs for both types of positional encoding.
Markers
(\typeIplus{} Type~$\mathrm{I}_+$, \typeII{} Type~$\mathrm{II}$,
\typeIminus{} Type~$\mathrm{I}_-$)
indicate parity.}
\label{fig:abcd-simplex-intro}
\end{figure}

Note that profiles of \emph{randomly} initialized bilinear forms accumulate
tightly at the mid-point $(\tfrac12,\tfrac12,0,0) \in \Theta$.
Trained bilinear forms generally have large effective rank (see \S\ref{sec:effectiveRank}) so accumulation near $\Theta$ 
is not explained by proximity to low-rank forms. 
In short we do not have 
a complete explanation of Observation~\ref{obs:head-profiles}.
The face $\Delta = \{a=0\}$ contains the profiles of symmetric forms.
Let $p \colon \Delta_3 \to \Delta$ denote the projection
    $p(a,b,c,d)=(0,a+b,c,d)$.
As a partial explanation,
we prove that $p(\Theta)$
is the limiting set of profiles in the large-dimensional limit
for any random ensemble of symmetric matrices
with proportional positive and negative eigenvalue distributions.

\begin{theorem}[Theorem~\ref{thm:accumulation}]\label{thm:accumulationIntro}
Let $X$ be a positive random variable with finite, nonzero second
moment.  Fix $u,v>0$ and set $\rho=u/v$.
For each $n$ let $N_n \geq 2n$ and let $S_n \in M_{N_n}(\RR)$
be a symmetric matrix with $n$ positive eigenvalues sampled from $vX$, $n$
negative eigenvalues sampled from $-uX$, and any remaining eigenvalues
equal to zero.
Then the profile $\pi(S_n)$ converges almost surely as $n \to \infty$ to
\[
    \frac{1}{1+\rho^2}
    \bigl(0,\,
    2\rho,\,
    \max(1-\rho,0)^2,\,
    \max(\rho-1,0)^2\bigr) \in p(\Theta).
\]
\end{theorem}

The hypothesis about proportional positive and negative eigenvalue distributions
is approximately satisfied by trained bilinear forms (Observation~\ref{obs:pairing}).

\subsection{Eigenvalues of trained bilinear forms}\label{subsec:parity}

Eigenvalue distributions of the bilinear forms
in the \texttt{Qwen3} models
were investigated by Darveshi~\cite{darveshi2025spectral};
they found that the absolute values of the eigenvalues
are well-modelled by a {gamma distribution}.
We performed the same experiment for the symmetric parts.
Theorem~\ref{thm:balancedIntro} suggests
the distribution should be {multimodal}:
nonzero eigenvalues occur with equal frequency
on either side of zero so the distribution should have at least two peaks.
For $\nu \geq 0$, let $\gamma_\nu^+$ denote
a gamma-like probability distribution on $[0,\infty)$ 
with mode $\nu$, and 
for $\mu \leq 0$ let $\gamma_\mu^-$ denote 
the reflection across $0$ of a gamma-like probability distribution, 
supported on $(-\infty,0]$ with mode $\mu$. 

\begin{observation}[Figure~\ref{fig:mean-spectrum}]\label{obs:bimodal-gamma}
Eigenvalues of the symmetric part $S$ of trained bilinear forms
have a bimodal distribution
well-modelled by a sum of two gamma-like distributions 
with positive and negative modes: 
\begin{equation}\label{eqn:bimodalGamma}
\lambda(S) \sim \tfrac{1}{2}(\gamma_\mu^- + \gamma_\nu^+),
\qquad \mu \leq 0 \leq \nu.
\end{equation}
For $550$ of the $9{,}512$ heads ($5.8\%$), one or both fitted
modes were zero. 
\end{observation}

\begin{figure}[!htbp]
\centering
\includegraphics[width=0.5\textwidth]{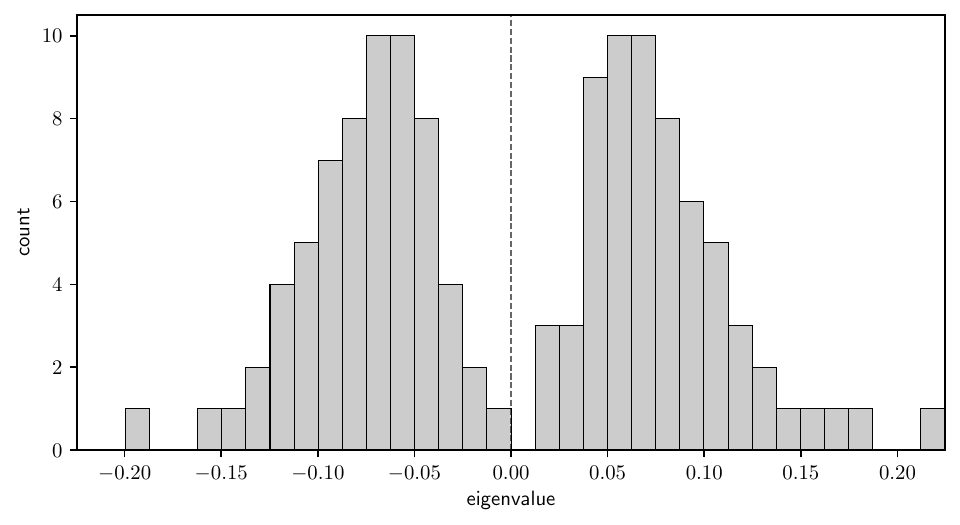}
\caption{The mean eigenvalue distribution of symmetric parts for \texttt{gpt2}.}
\label{fig:mean-spectrum}
\end{figure}

We observed that the modes $\mu,\nu$ of the bimodal eigenvalue distributions
predict which branch of $\Theta$ is nearest to the profile.
The midpoint $(\tfrac12,\tfrac12,0,0)$ is denoted $\Theta_\circ$,
the line segment from $\Theta_\circ$ to $c=1$ is denoted $\Theta_+$,
and the line segment from $\Theta_\circ$ to $d=1$ is denoted $\Theta_-$.
We call a head Type $\mathrm{I_+}$ if its profile is near $\Theta_+$,
Type $\mathrm{I_-}$ if the profile is near $\Theta_-$,
and otherwise call it Type $\mathrm{II}$ (see \S\ref{subsec:trained-types} for details).
We looked at the histogram of the mean sorted spectrum of
the ten most extreme heads in each cluster
(Figure~\ref{fig:spectra-cluster-barycenters})
and observed the following relationships between modes and the three profile types:
\begin{enumerate}
\item Type $\mathrm{I_+}$ (positive-definite-like): $|\mu| \ll \nu$;
\item Type $\mathrm{I_-}$ (negative-definite-like): $\nu \ll |\mu|$;
\item Type $\mathrm{II}$ (symmetric spectrum): $|\mu + \nu| \ll \min(|\mu|, \nu)$.
\end{enumerate}

\begin{figure}[!htb]
\centering
\includegraphics[width=\linewidth]{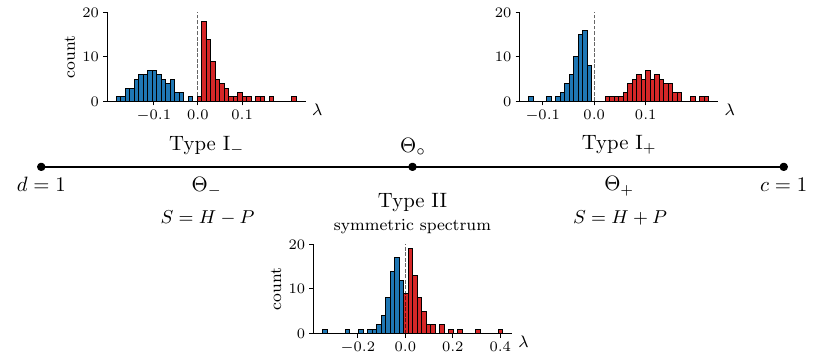}
\caption{Measured eigenvalue distributions in \texttt{gpt2}, attached to their
branches or midpoint of $\Theta$. Each histogram shows the mean of the sorted, unit-length spectra of
the ten most extreme heads in one cluster. Here $H$ has spectrum symmetric
about zero and $P$ is positive semidefinite.}
\label{fig:spectra-cluster-barycenters}
\label{fig:accumulation-family-intro}
\end{figure}

\subsubsection{Alignment with experimental classifiers}

To see whether there were other discerning features,
we performed principal component analysis
to obtain low-dimensional projections
of the eigenvalue distributions of the symmetric parts
of the trained attention heads:
sorting each distribution and rescaling it to unit length gives a point
of $\RR^{2n}$, and the leading coordinates of the singular value
decomposition of the resulting matrix project these to three dimensions.
The projection is plotted in Figure~\ref{fig:spectra-cluster-scatter3d}
for the $144$ heads of the \texttt{gpt2} model.
The results show that parity is well-aligned with
the experimentally obtained classifiers, 
so it may be helpful for classifying head behavior.
A second set of head classifiers following a different approach with moments
is described in Appendix~\ref{app:moments}.

\begin{figure}[!htbp]
\centering
\includegraphics[width=\textwidth]{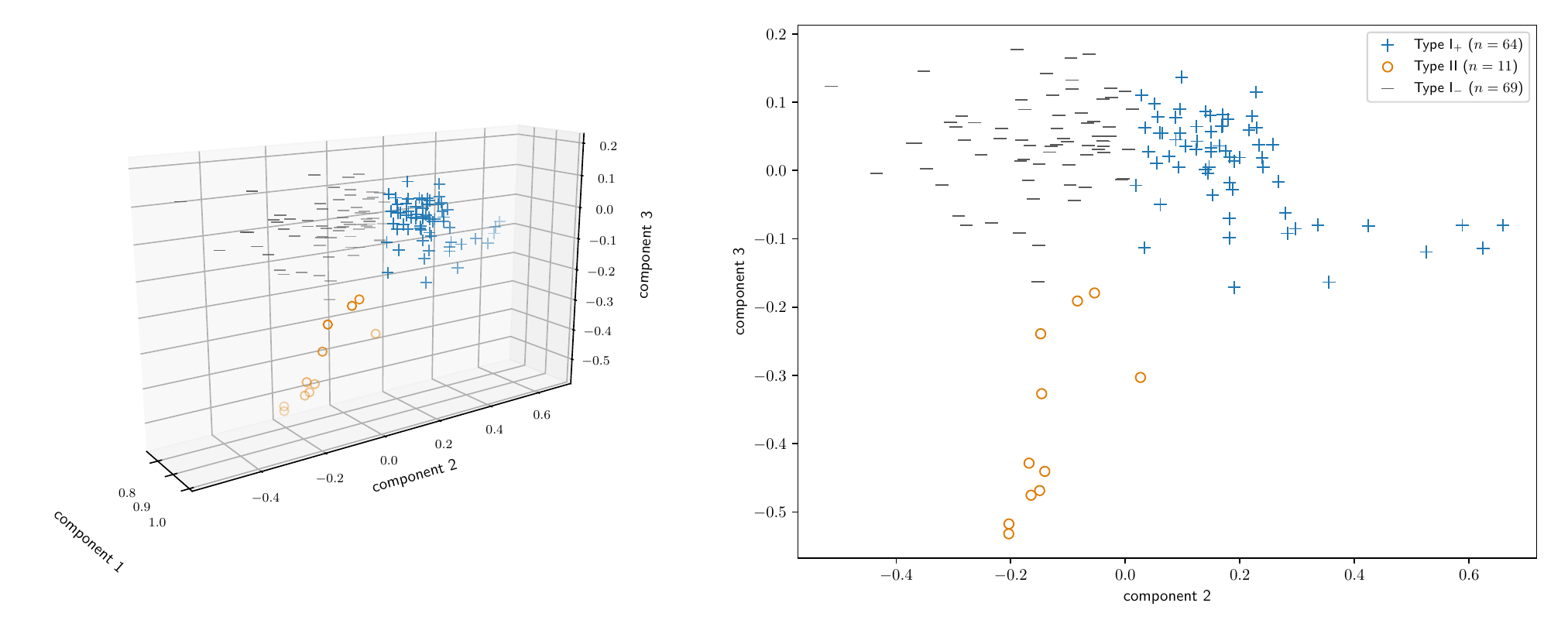}
\caption{Projection of \texttt{gpt2}'s $144$ symmetric part
eigenvalue distributions to the first three
principal components (left), and to the second and third components
(right).}
\label{fig:spectra-cluster-scatter3d}
\end{figure}

\subsection*{Notation}

Let $C$ be a finite-dimensional real vector space of dimension $N$.
Let $k$ be a nonnegative integer.
Let
\[
    \Delta_{k}
    =
    \left\{
        (p_0,\ldots,p_{k})\in \RR^{k+1}_{\geq 0}:
        \sum_{j=0}^k p_j=1
    \right\}
\]
be the standard $k$-simplex.
The identity map will be denoted $\id$.
The set of linear operators on $C$ will be denoted $\End(C)$.
The dual space of $C$,
i.e.\ the vector space of linear maps $C \to \RR$,
will be denoted $C^\vee$.
In what follows we make use of the canonical isomorphism
of vector spaces $C^\ell \cong C \otimes \RR^\ell$,
in particular writing $U \otimes V \in \End(C^\ell)$
for $U \in \End(C)$, $V \in \End(\RR^\ell)$.
The set of $N \times N$ real matrices is denoted by $M_N(\RR)$.
For a real matrix $L$ let $\|L\| = \sqrt{\mathrm{tr}(L^T L)}$
denote the Frobenius norm.

\section{Bilinear attention}\label{sec:bilinearAttention}

In this section we give an equivalent definition for attention
different from what is found in the literature.
We restrict to decoder-only transformers
and ignore layer normalization.

\subsection{Bilinear forms}\label{subsec:bilinear-forms}

We recall some basic facts about real bilinear forms.
A bilinear form on a real vector space $C$ is a function
$$B \colon C \times C \to \RR$$
which is linear in each argument separately,
i.e. for all $x,y,z \in C$ and $\lambda \in \RR$ we have
\begin{align*}
B(x+y,z) &= B(x,z) + B(y,z),\qquad B(\lambda x,y) = \lambda B(x,y),\\
B(x,y+z) &= B(x,y) + B(x,z),\qquad B(x,\lambda y) = \lambda B(x,y).
\end{align*}
If $B(x,y) = B(y,x)$ (resp. $-B(y,x)$) for all $x,y \in C$ then $B$ is called \defn{symmetric}
(resp. \defn{antisymmetric}).
Every bilinear form $B$ can be uniquely expressed
as a sum $B = S+T$
of a symmetric bilinear form $S$ and an antisymmetric bilinear form $T$,
where
$$S(x,y) = \tfrac12 \big(B(x,y) + B(y,x)\big),\qquad
T(x,y) = \tfrac12 \big(B(x,y) - B(y,x)\big).$$
Given a basis $e_1,\ldots,e_N$ of $C$,
the \defn{Gram matrix} of $B$ is the $N \times N$ matrix $L$
with entries $L_{ij} = B(e_i,e_j)$.
The \defn{radical} of $B$ is the subspace
$\{x \in C : B(x,y) = 0 \text{ for all } y \in C\}$.
A subspace $W\subset C$ is \defn{isotropic} for $B$ if
$B(x,y)=0$ for all $x,y \in W$.
A nondegenerate antisymmetric bilinear form is called
\defn{symplectic}.

\subsubsection{Sylvester's Law of Inertia}

If $B$ and $B'$ are bilinear forms and
there is an invertible linear map $f \colon C \to C$ such that
$B'(x,y) = B(f(x),f(y))$ for all $x,y \in C$ then $f$ is called an \defn{isomorphism}
between $B$ and $B'$.
Sylvester's Law of Inertia says that
a symmetric bilinear form $B$ on $C$
is determined up to isomorphism
by three non-negative integers $(n_+,n_-,n_0)$.
These are computed from the Gram matrix $L$ of $B$,
which is symmetric since $B$ is:
$n_+$ (resp. $n_-$) is the number of positive (resp. negative)
eigenvalues of $L$ counted with multiplicity, and
$n_0$ is the dimension of the kernel of $L$.
The tuple $(n_+,n_-,n_0)$ is called the \defn{signature} of $B$;
by Sylvester's law it is independent of the chosen basis.
The form $B$ is \defn{positive semidefinite} if $B(x,x) \geq 0$ for
every $x$ and \defn{negative
semidefinite} if $-B$ is positive semidefinite; it is
\defn{semidefinite} if it is one or the other.
If $n_+ = n_-$ then we say $B$ is \defn{balanced}.
A nondegenerate balanced form is called \defn{hyperbolic}.
Hyperbolicity requires equal numbers of positive and negative eigenvalues,
but does not require their magnitudes to agree. Symmetry of the spectrum
about zero is stronger than balance and depends on the chosen inner product.
By Sylvester's law, there is exactly one hyperbolic form on
$\RR^{2n}$ up to isomorphism, namely the form with Gram matrix
\[
    \begin{pmatrix} 0 & I \\ I & 0 \end{pmatrix} \in M_{2n}(\RR)
\]
where $I$ denotes the $n \times n$ identity matrix.

\subsection{Attention functions}

First we define a very general class of attention functions.

\begin{definition}
A smooth function $a\colon C^k\to C$ is an \defn{attention function}
if there are a smooth function
\[
    \omega=(\omega_{1},\ldots,\omega_{k})
    \colon C^k\to \Delta_{k-1},
\]
a smooth convex function $\phi \colon \RR^k \to \RR$ and
a polynomial map $\sigma \colon C^k \to \RR^k$
such that
$$
\omega = (\nabla \phi) \circ \sigma
\qquad \text{and} \qquad
    a(x_1,\ldots,x_k)
    =
    \sum_{j=1}^k
    \omega_{j}(x_1,\ldots,x_k)x_j.
$$
\end{definition}

The components of $\omega$
are called the \emph{attention weights},
$\sigma$ is the \emph{scoring function},
and $\phi$ we call the \emph{potential function};
with an abuse of terminology, we regard
$\omega,\sigma,\phi$ as part of the data of the attention function.

For the next three sections,
let $H$ be a real inner product space of dimension $n$ 
called the \defn{head space},
and let $Q,K \colon C \to H$ be independent linear maps
called the \emph{query and key maps}.

\subsection{QKV attention}\label{sec:QKVattention}

We explain how to recover
the definition of attention from \cite{vaswani2017attention}. 
First we construct the scoring function.
To encode positional information,
vectors $p_j \in C$ are fixed
for each $j = 1,\ldots,k$.
The score of $x_k$ on $x_j$ is
the $j$-th component of $\sigma$, given by
\[
\sigma_j(x_1,\ldots,x_k)
=
\left\langle Q(x_k+p_k),\, K(x_j+p_j)\right\rangle
/\sqrt{n}.
\]
Expanding, we get
\begin{equation}\label{eqn:scoreQKV}
\sigma_j(x_1,\ldots,x_k)
=
\frac{1}{\sqrt{n}}
\Big(
\langle Q x_k, K x_j\rangle
+
\left\langle Q x_k, K p_j\right\rangle
+
\left\langle Q p_k, K x_j\right\rangle
+
\langle Q p_k, K p_j\rangle
\Big),
\end{equation}
a polynomial function on $C^k$.
The potential is given by\footnote{To see that $\phi$ is convex,
first set $\omega_j=e^{s_j}/\sum_i e^{s_i}$ and $\omega = (\omega_1,\ldots,\omega_k)^T$.
Then
\[
    \frac{\partial \phi}{\partial s_j}=\omega_j,
    \qquad
    \frac{\partial^2 \phi}{\partial s_i\partial s_j}
    = \omega_i\delta_{ij}-\omega_i\omega_j,
\]
so $\operatorname{Hess}\phi=\operatorname{diag}(\omega)-\omega\omega^T$. For any
$v\in\RR^k$, Cauchy--Schwarz applied to $(\sqrt{\omega_i})_i$ and
$(\sqrt{\omega_i}\,v_i)_i$ gives
\[
    \Big(\sum_i \omega_iv_i\Big)^2
    =
    \Big(\sum_i \sqrt{\omega_i}\cdot\sqrt{\omega_i}\,v_i\Big)^2
    \leq
    \Big(\sum_i \omega_i\Big)\Big(\sum_i \omega_iv_i^2\Big)
    =
    \sum_i \omega_iv_i^2.
\]
Therefore
\[
    v^T\operatorname{Hess}\phi\,v
    =
    \sum_i \omega_iv_i^2-\Big(\sum_i \omega_iv_i\Big)^2
    \geq 0,
\]
so $\phi$ is convex by the Hessian criterion for convexity.}
\begin{equation}\label{eqn:softmaxPotential}
    \phi(s_1,\ldots,s_k)
    =
    \log(e^{s_1}+\cdots+e^{s_k}).
\end{equation}
The attention weights $\omega$ formed from $\phi$ and $\sigma$
recover QKV attention.\footnote{Strictly speaking we have recovered ``QK attention'' with trivial value map $V = \id$. In our formulation of attention, value maps come at a later stage; see \S\ref{sec:transformers}.}

\subsection{RoPE attention}\label{sec:RoPE}

QKV attention encodes positional information in the linear
and constant terms of the scores \eqref{eqn:scoreQKV} and the bilinear
term is \emph{independent} of position.
Rotary position embedding (RoPE)
\cite{su2024roformer}
is a variant of QKV attention with no linear or constant terms
with all positional information in the bilinear terms.
Assume $n = 2m$ is even and
let $R \colon H \to H$ be an {orthogonal} transformation
whose eigenvalues are
\[
    e^{\pm i \omega q^{r-1}},
    \qquad
    r = 1, \ldots, m,
\]
for some fixed $\omega \in (0, \pi)$ and $0 < q < 1$ 
(in \cite{su2024roformer} one takes $\omega = 1$ and
$q = 10000^{-1/m}$).
In RoPE attention,
the score of $x_k$ on $x_j$ is defined by
\begin{equation}\label{eqn:ropeScore}
    \sigma_j(x_1,\ldots,x_k)
    =
\langle Kx_j,\;R^{k-j}Qx_k\rangle/\sqrt{n}.
\end{equation}

\subsection{ALiBi attention}\label{sec:ALiBi}

Attention with linear biases (ALiBi) \cite{press2022train}
encodes positional information in the constant terms of the scores.
Fix a positive scalar $\mu$ for each head.
In ALiBi attention, the score of $x_k$ on $x_j$ is defined by
\begin{equation}\label{eqn:alibiScore}
    \sigma_j(x_1,\ldots,x_k)
    =
\langle Kx_j,\;Qx_k\rangle/\sqrt{n}
    -\mu(k-j).
\end{equation}
The scalar $\mu$ is fixed before training, and the term $-\mu(k-j)$
penalizes attention to more distant inputs.
There are no linear terms in the input vectors, and the bilinear term
is independent of position.

\subsection{Biaffine attention}

We have seen that QKV attention encodes position in
the linear and constant terms of the score and
uses the same bilinear form at different distances;
RoPE attention has no linear nor constant terms and
encodes positional information
by using distinct bilinear forms at different distances;
ALiBi uses the same bilinear form at every distance and encodes
position entirely in the constant terms.
The bilinear, linear, and constant parts of a score are separated
in Lemma~\ref{lem:biaffine-decomposition} below.
By lifting the constraints on the bilinear terms in RoPE
and allowing for lower-order terms
we obtain the following more general class of attention functions.

\begin{definition}
Let $a \colon C^k \to C$ be an attention function with scoring function
    $\sigma\colon C^k \to \RR^k$ of the form
$$
\sigma(x_1,\ldots,x_{k}) =
\big(P_{k-1}(x_1,x_k),
P_{k-2}(x_2,x_k),
\ldots,
P_1(x_{k-1},x_k),
P_0(x_k,x_k)
\big)
$$
for polynomials $P_j \colon C^{2} \to \RR$.
By an abuse of terminology, 
if the polynomials $P_j$ all have bidegree at most $(1,1)$
then $a$ is called \defn{biaffine}.
If the biaffine polynomials are bilinear forms then
$a$ is \defn{bilinear}.
\end{definition}

The following elementary lemma shows that the bilinear, linear, and
constant components of a biaffine score are well-defined; we use it to
attach a bilinear form to each attention head in
\S\ref{sec:experimentsQK}.

\begin{lemma}\label{lem:biaffine-decomposition}
Every polynomial
\(P:C\times C\to \RR\) of bidegree at most \((1,1)\) admits a unique
decomposition
\[
    P(u,v)=B(u,v)+L(u)+R(v)+D,
\]
where \(B\) is bilinear, \(L\) and \(R\) are linear, and \(D\in\RR\).
The components are given by
\[
    D=P(0,0), \,\,\,
    L(u)=P(u,0)-P(0,0),\,\,\,
    R(v)=P(0,v)-P(0,0),
\]
and
    $B(u,v)=P(u,v)-P(u,0)-P(0,v)+P(0,0)$.
\end{lemma}

\subsubsection{RoPE-like attention functions}\label{sec:compositional}

In \cite[\S3.4.1]{su2024roformer} there is an argument given to justify
their choice of bilinear form
(cf.\ \eqref{eqn:ropeScore} with $d = k - j$)
$$P_d(x,y) = \langle K x, R^d Q y\rangle/\sqrt{n}.$$
We give another derivation of their formula and generalize it
to a larger class of bilinear attention functions.
Let $a\colon H^k \to H$ be a bilinear attention function
with scoring function $\sigma \colon H^k \to \RR^k$ given by
\[
    \sigma(u_1, \ldots, u_k)
    = \big(B_{k-1}(u_1, u_k),\,
    B_{k-2}(u_2, u_k),\,
    \ldots,\,
    B_0(u_k, u_k)\big)
\]
for bilinear forms $B_{k-1}, \ldots, B_0$ on $H \times H$.
Each bilinear form $B_j$ determines a map
$L_j \colon H \to H^\vee$ defined by
$L_j(v)(u) = B_j(u,v)$.
Assume that $B_0$ is nondegenerate.
Define the linear operator
$$T_d = L_0^{-1} L_d \in \End(H).$$

\begin{definition}
If $T_{d+e} = T_d T_e$ for all integers $d,e\geq 0$ with $d+e < k$
then $a$ is called \defn{RoPE-like}.
\end{definition}

The next proposition classifies RoPE-like bilinear attention functions.

\begin{proposition}
A bilinear attention function $a$ as above with nondegenerate $B_0$
is RoPE-like
if and only if
there exists a linear map $S \in \End(H)$
such that $T_d = S^d$ for all $0 \leq d < k$.
\end{proposition}

Note: when $S$ exists it is unique and given by
$S = T_1= L_0^{-1} L_1$.

\begin{proof}
Let $a$ be a RoPE-like bilinear attention function.
Then $L_{d+e} = L_d L_0^{-1} L_e$
for all $d,e \geq 0$ such that $d+e < k$; in particular, this shows that
$L_{d+1} = L_d S$.
By induction, we conclude that
$L_{d} = L_0 S^d$ and therefore $T_d = S^d$.
The other direction is trivial.
\end{proof}

For RoPE itself, the head-space forms are
$B_d(u,v) = \langle u, R^d v\rangle/\sqrt{n}$.
The form $B_0$ is a scaled inner product, so it is nondegenerate,
and $L_d=L_0R^d$. Thus $T_d=R^d$, and these forms define
a RoPE-like attention function on $H$ with $S=R$.
The scores on $C$ in \S\ref{sec:RoPE} are recovered by setting
    $P_d(x,y)=B_d(Kx,Qy)$.

\subsection{Rank reduction}\label{sec:rankreduction}

It is crucial to parametrize through
families of {low-rank} attention functions.
Let $E$ be a real finite-dimensional vector space. 
Let $F \colon C^k \to E^k$ be a polynomial mapping. 
We have a family of attention functions on $C^k$
given by ``pulling back along $F$''
\begin{align}\label{eqn:subfamily}
\left\{\text{attention functions $E^k \to E$}\right\}
&\to
\left\{\text{attention functions $C^k \to C$}\right\} \\
a=(\sigma,\phi) &\mapsto F^\ast a =(\sigma \circ F,\phi).
\end{align}

\subsubsection{The standard construction}\label{sec:standardConstruction}

Set $E = H \oplus H$,  
let $\mu\geq 0$ and let $R \in \End(H)$ be an orthogonal transformation. 
Set 
$$
\sigma((q_1,k_1),\ldots,(q_k,k_k)) = 
(\langle R^{k-1} q_k, k_1 \rangle, \langle R^{k-2}q_k, k_2\rangle, \ldots,
\langle q_k, k_k \rangle)/\sqrt{n} + c,
\qquad c=-\mu(k-1,k-2,\ldots,0).
$$
This defines a polynomial map $\sigma \colon E^k \to \RR^k$. 
Let 
    $\binom{Q}{K} \colon C \to E: 
    x \mapsto (Qx, Kx)$ 
be a linear map. 
Let $p \in C^k$ and let $F \colon C^k \to E^k$ be given by 
$$
F(x_1,\ldots,x_k) = 
((Qy_1,Ky_1),(Qy_2,Ky_2),\ldots,(Qy_k,Ky_k)),
\qquad y = x+p.
$$
The pullback of $\sigma$ by $F$ gives a family of attention functions on $C^k$ 
including QKV, RoPE, and ALiBi attention functions as special cases. 

\subsection{Potentials}

To orient the reader we give some examples of potential functions.

\medskip\noindent\emph{Softmax potential.}
Let
    $\phi(s_1,\ldots,s_k)
    =
    \log(e^{s_1}+\cdots+e^{s_k})$.
The attention function is given by
\[
a(x_1,\ldots,x_k) = \omega_1 x_1 + \cdots + \omega_k x_k
\quad\text{where}\quad
    \omega_{j}
    =
    \frac{
        e^{\sigma_j(x)}
    }{
        \sum_{i=1}^k e^{\sigma_i(x)}
    }.
\]

\medskip\noindent\emph{Linear potential.}
As a degenerate case we take a linear potential function $\phi$.
Then $\omega = \nabla \phi \circ \sigma = \nabla \phi$
is independent of the scoring function.
In order for the weights to be simplex-valued,
the potential must be given by
$\phi(s) = p_0 \cdot s$
for some $p_0 \in \Delta_{k-1}$.
Upon identifying $C \otimes \RR^k = C^k$
the attention function is given by
$$a(x_1,\ldots,x_k) = (\id \otimes\, \phi)(x_1,\ldots,x_k).$$
We give two examples of linear potentials.
Let
   $\phi
    =
    \tfrac{1}{k}(s_1+\cdots+s_k)$.
The associated attention function is
\[
    a(x_1,\ldots,x_k)
    =
    \tfrac1k(x_1+\cdots+x_k).
\]
Next fix $m\in\{1,\ldots,k\}$ and let
    $\phi(s_1,\ldots,s_k)=s_m$.
Then $\nabla\phi=e_m$, the $m$-th standard basis vector, so
\[
    a(x_1,\ldots,x_k)=x_m.
\]

\begin{remark}[Max potential]\label{rmk:hardmax}
Softmax potential is a smoothed form of the max function:
for $\tau > 0$ let
\[
    \phi_\tau(s_1,\ldots,s_k)
    =
    \tau\log\!\left(e^{s_1/\tau}+\cdots+e^{s_k/\tau}\right).
\]
This potential is smooth and convex with gradient
\[
    (\nabla\phi_\tau(s))_j
    =
    \frac{e^{s_j/\tau}}{\sum_i e^{s_i/\tau}}.
\]
At $\tau=1$ this is softmax, and as $\tau \to 0^+$, $\phi_\tau$ converges pointwise to the max potential
\[
    \phi_0(s_1,\ldots,s_k)
    =
    \max(s_1,\ldots,s_k).
\]
On the open subset where $\phi_0$ is differentiable,
the attention function outputs the token
at the position with the largest score.
While we have not attempted to do so,
it would be natural and straightforward to relax
the definition of attention functions to include such mildly nondifferentiable potentials.
\end{remark}

\subsection{Attention heads}\label{sec:attentionblock}

\begin{definition}
For each integer $k = 2,\ldots,\ell$
let $a_k \colon C^k \to C$ be an attention function.
The \emph{(causal) attention head $A \colon C^\ell \to C^\ell$ associated to $a_2,\ldots,a_\ell$} is
\begin{equation}\label{eqn:fullattn}
    A(x_1,\ldots,x_\ell)
    = \left(x_1,a_2(x_1,x_2),\ldots,a_\ell(x_1,\ldots,x_\ell)\right).
\end{equation}
(Causality is enforced by the independence of $a_k$ from $x_j$ for $j > k$.)
\end{definition}

\begin{example}
We use the softmax potential
$\phi(s_1,\ldots,s_k)=\log(e^{s_1}+\cdots+e^{s_k})$.
The length two ($k=2$) bilinear attention head is
$$
A(x,y) =
\big(
x, y + (x-y) \sigma
\big)
\quad\text{where}\quad
0 < \sigma = \frac{e^{b}}{1 + e^{b}} < 1,
\quad b = B_1(x,y).
$$
The length three ($k=3$) bilinear attention head is
$$
A(x,y,z) =
\big(
x,\, y + (x-y)\sigma,\, z + (x-z)\sigma_1 + (y-z)\sigma_2
\big)
$$
where $\sigma = e^{b}/(1+e^{b})$, $b = B_1(x,y)$ as above, and
$$
\sigma_1
=
\frac{e^{b_2}}{1+e^{b_1}+e^{b_2}},
\qquad
\sigma_2
=
\frac{e^{b_1}}{1+e^{b_1}+e^{b_2}},
\qquad
b_1 = B_1(y,z),
\quad
b_2 = B_2(x,z).
$$
\end{example}

\subsection{Transformers}\label{sec:transformers}

Transformers are a type of piecewise-smooth function
$T \colon C^\ell \to C^\ell$
operating on sequences of elements in $C$.
Let $A_1,\ldots,A_h\colon C^\ell \to C^\ell$ be attention heads
and let $V_1,\ldots,V_h \colon C \to C$ be linear maps, called \defn{value maps}.
The \defn{attention block associated to $\{(A_i,V_i)\}_{i=1}^h$}
is
$$A = \sum_{i=1}^h (V_i \otimes \id) \circ A_i.$$

Let $A\colon C^\ell \to C^\ell$ be an attention block
and let $M \colon C \to C$ be a piecewise smooth function.
The \emph{attention component}
$T_A  \colon C^{\ell} \to C^{\ell}$
is defined as
$$
T_{A}= \id \otimes \id + A,
$$
or explicitly, if $a_1,\ldots,a_\ell$ are the components of $A$, then
$$T_A(x_1,\ldots,x_\ell) =
(x_1 + a_1(x_1),
x_2 + a_2(x_1, x_2),
\ldots,
x_\ell + a_\ell(x_1,\ldots,x_\ell)).$$
The \defn{neural network component}
$T_M  \colon C^{ \ell} \to C^{ \ell}$
is defined as
$$
T_M = \id \otimes \id + M\otimes \id,
$$
or explicitly,
$$
T_M(y_1,\ldots,y_\ell) = (y_1 + M(y_1), y_2 + M(y_2), \ldots, y_\ell + M(y_\ell)).
$$

\begin{definition}
The \defn{transformer $T \colon C^{ \ell} \to C^{ \ell}$
associated with $A,M$}
is the function given by
$$
T = T_M \circ T_A = (\id \otimes \id + M\otimes \id) \circ (\id \otimes \id + A).
$$
Note that the neural network $M$ acts in parallel (component-wise)
while cross-term interactions only arise from the attention block.
\end{definition}

\subsubsection{Transformers are unipotent}

The presence of the so-called ``residual connection'' is
a significant design decision which has the effect of
making transformers near-identity or ``unipotent''. 
(In our formulation, the residual connection corresponds 
to the $\id \otimes \id$ summand in $T_A$ and $T_M$.)
The next proposition makes this precise with a basic estimate.

\begin{proposition}\label{prop:near-identity}
Fix a norm $\|\cdot\|$ on $C$, and for
$x = (x_1, \ldots, x_\ell) \in C^\ell$ set
$\|x\|_\infty = \max_k \|x_k\|$.
Let $T$ be the transformer associated to an attention block
$A = \sum_{i=1}^h (V_i \otimes \id) \circ A_i$ and a neural network
component $M$, and suppose that $\|M(y)\| \leq L\, \|y\|$ for all
$y \in C$.
Set $V = \sum_{i=1}^h \|V_i\|$, where $\|V_i\|$
denotes the operator norm.
Then for all $x \in C^\ell$,
\[
    \|T(x) - x\|_\infty
    \leq
    \bigl( L + V + LV \bigr)\, \|x\|_\infty .
\]
\end{proposition}

\begin{proof}
Since the attention weights take values in the simplex
$\Delta_{k-1}$, every attention function satisfies
\[
    \|a(x_1, \ldots, x_k)\|
    = \Bigl\| \sum_{j=1}^k \omega_j x_j \Bigr\|
    \leq \sum_{j=1}^k \omega_j \|x_j\|
    \leq \max_{j} \|x_j\|,
\]
so every component of an attention head $A_i$ has norm at most
$\|x\|_\infty$, the first component being $x_1$ itself.
The components of the attention block therefore satisfy
\[
    \|A(x)_k\|
    = \Bigl\| \sum_{i=1}^h V_i \bigl( A_i(x)_k \bigr) \Bigr\|
    \leq \sum_{i=1}^h \|V_i\| \, \|x\|_\infty
    = V \|x\|_\infty .
\]
Writing $T = T_M \circ T_A$, the $k$-th component of $T(x) - x$ is
\[
    A(x)_k + M\bigl( x_k + A(x)_k \bigr),
\]
whose norm is at most
\[
    V \|x\|_\infty + L \bigl( \|x\|_\infty + V \|x\|_\infty \bigr)
    = \bigl( (1+L)(1+V) - 1 \bigr) \|x\|_\infty .
    \qedhere
\]
\end{proof}

\begin{example}
For a fully connected neural network $M(y) = W_2\, \rho(W_1 y)$ whose
activation $\rho$ is $1$-Lipschitz and fixes $0$ (e.g.~RELU), the hypothesis
holds with $L = \|W_2\| \, \|W_1\|$.
\end{example}

\subsubsection{Standard construction (multiple heads)}

The standard construction (\S\ref{sec:standardConstruction}) 
extends naturally to multiple heads. 
For each $i = 1,\ldots,h$ let $A_i \colon C^\ell \to C^\ell$ 
be an attention head constructed from attention functions 
using the standard construction, 
and let $V_i \colon C \to C$ be a linear map 
factoring through $H$.\footnote{This map is typically written $V_i = W_{O,i} W_{V_i}$ where
$W_{V_i}\colon C \to H$ and
$W_{O,i} \colon H\to C$ is the ``output map''.}
The \defn{multi-head attention block associated to
$\{(A_i,V_i)\}_{i=1}^h$} is the attention block
$A=\sum_{i=1}^h (V_i \otimes \id) \circ A_i$.

\subsection{Parameter counts}

We compare the counts of trainable parameters
for the variants of attention introduced above.
The potential function \(\phi\) is regarded as fixed.
The positional vectors \(p_j\) in QKV attention
are also regarded as fixed,
as are the orthogonal transformation in RoPE attention and the
$\mu$ parameter in ALiBi attention.
For biaffine and bilinear attention we assume the scoring
polynomials depend only on the distance between inputs:
the attention functions $a_2,\ldots,a_\ell$ of the head share a
single list of polynomials $P_0,\ldots,P_{\ell-1}$, with $a_k$
using $P_{k-1},\ldots,P_0$.
A biaffine attention head on $C^\ell$ is then determined by $\ell$
biaffine maps.
The value maps enter the attention block in the same way for every variant 
so we have excluded their parameters. 

\begin{table}[h]
\centering
\renewcommand{\arraystretch}{2}
\setlength{\tabcolsep}{10pt}
\begin{tabular}{@{}lc@{}}
Model & Scoring parameters ($V=\id$) \\
\midrule
QKV/RoPE/ALiBi attention
    & \(2N^2\) \\
Biaffine attention
    & \(\ell(N+1)^2\) \\
Bilinear attention
    & \(\ell N^2\) \\
Multi-head QKV/RoPE/ALiBi attention
    & \(2N^2\) \\
Multi-head bilinear
    & \(\dfrac{\ell N^2}{h}\) \\[1em]
\end{tabular}
\caption{Trainable scoring parameter counts in the attention blocks;
value map and positional encoding parameters are excluded. Here \(N=\dim C\),
\(\ell\) is the sequence length, and \(h\) is the number of heads.
The last row uses fixed coordinate projections as described above.}
\label{tab:parameter-counts}
\end{table}

\subsection{Bilinear attention recovers QKV attention}\label{sec:homogenization}

Although QKV attention introduces linear and constant terms
due to the positional embedding vectors, it can be realized as bilinear attention
using a basic construction in projective geometry called homogenization.

\begin{proposition}
\label{prop:qkv-homogenizes-to-bilinear}
Let $\widehat C=C\oplus\RR e$ and let
$\iota\colon C\to\widehat C$, $x\mapsto x+e$.
For every single-head QKV transformer $T\colon C^\ell\to C^\ell$
there is a single-head bilinear transformer
$\widehat T\colon\widehat C^\ell\to\widehat C^\ell$
with the same potential such that
    $\widehat T(\iota x_1,\ldots,\iota x_\ell)
    =
    \iota^\ell\bigl(T(x_1,\ldots,x_\ell)\bigr)$
for all $(x_1,\ldots,x_\ell)\in C^\ell$.
\end{proposition}

\begin{proof}
We will show that every QKV transformer on \(C\)
is realized on the affine chart \(\iota(C)=C+e\subset \widehat C\) by a bilinear
transformer on \(\widehat C\).
For a QKV head with query and key maps $Q, K \colon C \to H$
and fixed positional encoding vectors \(p_1,\ldots,p_\ell \in C\),
the score of the \(k\)-th input $x_k$ on the \(j\)-th input $x_j$ is
\[
\begin{aligned}
    s_{kj}(x_j,x_k)
    &=
        \frac{
        \left\langle Q(x_k+p_k),K(x_j+p_j)\right\rangle
    }{\sqrt{n}}\\
    &=
    \frac{1}{\sqrt{n}}
    \Big(
        \langle Qx_k,Kx_j\rangle
        +
        \langle Qx_k,Kp_j\rangle
        +
        \langle Qp_k,Kx_j\rangle
        +
        \langle Qp_k,Kp_j\rangle
    \Big).
\end{aligned}
\]
Define a bilinear form
    $\widehat B_{k-j,j}$ on $\widehat C$ by
\[
\begin{aligned}
    \widehat B_{k-j,j}(u+\alpha e,v+\beta e)
    \coloneqq
    \frac{1}{\sqrt{n}}
    \Big(
        \langle Qv,Ku\rangle
        +
        \alpha \langle Qv,Kp_j\rangle
        +
        \beta \langle Qp_k,Ku\rangle
        +
        \alpha\beta \langle Qp_k,Kp_j\rangle
    \Big).
\end{aligned}
\]
Then
    $\widehat B_{k-j,j}(\iota x_j,\iota x_k)
    =
    s_{kj}(x_j,x_k)$
so the bilinear scores of the homogenized attention head agree
with the original QKV scores on the affine chart \(C+e\).  Since the
score vectors are equal, the attention weights are equal.
If the original attention function is
    $a_k(x_1,\ldots,x_k)
    =
    \sum_{j=1}^k \omega_j x_j$
then the homogenized bilinear attention function satisfies
\[
\begin{aligned}
    \widehat a_k(\iota x_1,\ldots,\iota x_k)
    &=
    \sum_{j=1}^k \omega_j \iota(x_j) \\
    &=
    \sum_{j=1}^k \omega_j(x_j+e) \\
    &=
    \left(\sum_{j=1}^k \omega_jx_j\right)
    +
    \left(\sum_{j=1}^k\omega_j\right)e \\
    &=
    a_k(x_1,\ldots,x_k)+e \\
    &=
    \iota(a_k(x_1,\ldots,x_k)).
\end{aligned}
\]

Extend the value map and the neural network by ignoring the new coordinate:
    $\widehat V(u+\alpha e)\coloneqq V(u), \widehat M(u+\alpha e)\coloneqq M(u)$.
For any \(x,z\in C\) we have
    $\iota(x)+\widehat V(\iota z)
    =
    \iota(x+Vz)$
and
    $\iota(x)+\widehat M(\iota x)
    =
    \iota(x+M(x))$
so the transformer satisfies
    $\widehat T(\iota x_1,\ldots,\iota x_\ell)
    =
    \iota^\ell\bigl(T(x_1,\ldots,x_\ell)\bigr)$.
\end{proof}

\FloatBarrier

\section{Spectral experiments on bilinear forms}\label{sec:experimentsQK}

In this section we inspect the bilinear forms in the attention heads of
open-source large language models.

\subsection{Models}

We analyze fourteen publicly available pretrained language models,
comprising $9{,}512$ attention heads.  The models have head dimensions
$64$, $80$, or $128$ and use learned absolute positional embeddings,
ALiBi, or rotary position embeddings (RoPE):
\begin{itemize}
    \item \texttt{distilgpt2}~\cite{sanh2019distilbert}: distilled
    \texttt{gpt2}, $6$ layers, $12$ heads per layer ($72$ heads),
    $N=768$, and $n=64$.
    \item \texttt{gpt2}~\cite{radford2019language}: GPT-2 small,
    $12$ layers, $12$ heads per layer ($144$ heads), $N=768$, and $n=64$.
    \item \texttt{gpt2-medium}~\cite{radford2019language}:
    $24$ layers, $16$ heads per layer ($384$ heads), $N=1024$, and $n=64$.
    \item \texttt{gpt2-large}~\cite{radford2019language}:
    $36$ layers, $20$ heads per layer ($720$ heads), $N=1280$, and $n=64$.
    \item \texttt{opt-125m}~\cite{zhang2022opt}: OPT decoder-only,
    $12$ layers, $12$ heads per layer ($144$ heads), $N=768$, and $n=64$.
    \item \texttt{opt-350m}~\cite{zhang2022opt}:
    $24$ layers, $16$ heads per layer ($384$ heads), $N=1024$, and $n=64$.
    \item \texttt{opt-1.3b}~\cite{zhang2022opt}:
    $24$ layers, $32$ heads per layer ($768$ heads), $N=2048$, and $n=64$.
    \item \texttt{opt-2.7b}~\cite{zhang2022opt}:
    $32$ layers, $32$ heads per layer ($1{,}024$ heads), $N=2560$, and $n=80$.
    \item \texttt{opt-6.7b}~\cite{zhang2022opt}:
    $32$ layers, $32$ heads per layer ($1{,}024$ heads), $N=4096$, and $n=128$.
    \item \texttt{opt-13b}~\cite{zhang2022opt}:
    $40$ layers, $40$ heads per layer ($1{,}600$ heads), $N=5120$, and $n=128$.
    \item \texttt{bloom-3b}~\cite{bigscience2022bloom}: $30$ layers, $32$ heads per layer
    ($960$ heads), $N=2560$, $n=80$, and ALiBi positional encoding.
    \item \texttt{bloom-7b1}~\cite{bigscience2022bloom}: $30$ layers, $32$ heads per layer
    ($960$ heads), $N=4096$, $n=128$, and ALiBi positional encoding.
    \item \texttt{pythia-70m-deduped}~\cite{biderman2023pythia}: GPT-NeoX
    family, $6$ layers, $8$ heads per layer ($48$ heads), $N=512$,
    $n=64$, and RoPE on the first $25\%$ of each head's dimensions.
    \item \texttt{Mistral-Small-24B-Base-2501}~\cite{mistralai2025small3}: $40$ layers,
    $32$ query heads per layer ($1{,}280$ heads), $8$ key--value heads,
    $N=5120$, $n=128$, and RoPE.
\end{itemize}

As we saw in \S\ref{sec:QKVattention},
QKV attention is the special case of biaffine attention
where the bilinear components
(in the sense of Lemma~\ref{lem:biaffine-decomposition})
of the scoring function all equal the same bilinear form.
In particular, a QKV attention head $A$
is associated with a {single} bilinear form $B_A$ on $C^2$.
In the notation of \cite{vaswani2017attention}
where $C = \RR^N$ and $H = \RR^n$
and $Q,K \in M_{n\times N}(\RR)$,
this bilinear form is given by
$B_A (x,y)= x^T K^T Q y.$
We write $L_A$ for the Gram matrix $K^T Q$.
For \texttt{pythia-70m-deduped} and
\texttt{Mistral-Small-24B-Base-2501}, which use RoPE attention, the
form $L_A = K^T Q$ reported in the cross-model comparisons is the
base form at relative position $d=0$.

\subsection{Symmetry}\label{subsec:low-rank-bilinear-forms}

For each head $A$
	we asked whether the matrix $L=L_A$ was biased
	towards either symmetry or skew-symmetry.
Let $\|L\|$ denote the Frobenius norm.
Define the \defn{symmetric energy}
\[
    \mathscr{S}(A) := \frac{\|\frac{1}{2}(L + L^T)\|^2}{\|L\|^2}.
\]
The expected value of the symmetric energy for
a product $L = K^T Q$ where $K,Q$ are $n \times N$ random
matrices with iid standard normal entries is $1/2 + 1/(2N)$
(Proposition~\ref{prop:expected-symmetric-energy-low-rank-qk}).

\begin{observation}\label{obs:symmetric-bias}
Attention heads show a bias for symmetric forms: their symmetric
energies $\mathscr{S}$ lie
predominantly above the random baseline
(Figure~\ref{fig:qk-multimodel}); the per-model share of heads above it
ranges from $76\%$ to $99\%$.
\end{observation}

This computation corroborates the findings
of \cite{saponati2025underlying}.
Some heads are almost entirely symmetric.

\subsection{Pairing score}\label{sec:pairingScore}

Write $S=\tfrac12(L+L^T)$. Let $\lambda_+$ (resp. $\lambda_-$) denote
the vector of positive (resp. absolute values of negative) eigenvalues of $S$
sorted in decreasing order.
By appending zeros as needed, we will assume
$\lambda_+$ and $\lambda_-$ have the same length.
Define the \defn{pairing score}
\[
    \mathscr{P}(A) := 1 - \frac{\|\lambda_+-\lambda_-\|}{\|S\|}.
\]
The pairing score is one if and only if
the set of eigenvalues is symmetric about zero.
It is determined by the inner product of the two lists, and is
bounded in terms of the ratio of their norms.

\begin{lemma}\label{lemma:pairing-bound}
Let $S$ be a symmetric bilinear form with
$\lambda_+ \neq 0$, and let $\rho = \|\lambda_-\| / \|\lambda_+\|$.
Then
\[
    \mathscr{P}
    = 1 - \sqrt{1 - \frac{2 \langle \lambda_+, \lambda_- \rangle}{\|S\|^2}}
    \;\leq\;
    1 - \frac{|1 - \rho|}{\sqrt{1 + \rho^2}},
\]
with equality if and only if $\lambda_- = \rho\, \lambda_+$.
\end{lemma}

\begin{proof}
Since $\|S\|^2 = \|\lambda_+\|^2 + \|\lambda_-\|^2$, we have
$\|\lambda_+ - \lambda_-\|^2 = \|S\|^2 - 2 \langle \lambda_+, \lambda_- \rangle$,
which gives the equality.
By the triangle inequality,
$\|\lambda_+ - \lambda_-\| \geq \bigl| \|\lambda_+\| - \|\lambda_-\| \bigr|
= |1 - \rho|\, \|\lambda_+\|$, with equality if and only if one of
$\lambda_\pm$ is a nonnegative multiple of the other, that is,
$\lambda_- = \rho\, \lambda_+$.
Dividing by
$\|S\|^2 = (1 + \rho^2)\, \|\lambda_+\|^2$
gives the inequality.
\end{proof}

\begin{observation}\label{obs:pairing}
The positive eigenvalues and the absolute values of the negative
eigenvalues of each symmetric part, listed in decreasing order, are
approximately proportional.
Set $\rho = \|\lambda_-\| / \|\lambda_+\|$.
The relative error
\[
    \frac{\|\lambda_- - \rho\,\lambda_+\|}{\|\lambda_-\|}
\]
has median $0.13$ over the heads, with $80\%$ between $0.05$ and
$0.33$; the per-model median lies between $0.08$ and $0.18$.
Consistently, the pairing score is close to the bound of
Lemma~\ref{lemma:pairing-bound}, whose equality case is proportionality:
the per-model median gap between that bound and $\mathscr{P}$ lies
between $0.017$ and $0.041$.
\end{observation}

Figure~\ref{fig:qk-random-baseline} compares the pairing scores of
trained heads with a Gaussian QK baseline, and
Figure~\ref{fig:pairing-histogram} shows their distributions by model.
The per-model median pairing score lies between $0.77$ and $0.87$.

\subsection{The parity of trained heads}
\label{subsec:trained-types}

The midpoint $(\tfrac12,\tfrac12,0,0)$ is denoted $\Theta_\circ$,
the line segment from $\Theta_\circ$ to $c=1$ is denoted $\Theta_+$,
and the line segment from $\Theta_\circ$ to $d=1$ is denoted $\Theta_-$.
For a bilinear form $L$ with nonzero symmetric part $S$, write
$\pi(S)=(0,b_S,c_S,d_S)$.
We use the neighborhood
\[
    \mathscr N = \{(0,b,c,d)\in\Delta : c+d\leq\tfrac1{50}\}
\]
of the vertex $(0,1,0,0)$ in the face $\Delta=\{a=0\}$.
We call $L$ \defn{Type~$\mathrm{II}$} if $\pi(S)\in\mathscr N$.
Outside $\mathscr N$, we call it \defn{Type~$\mathrm{I}_+$}
if $c_S>d_S$, and \defn{Type~$\mathrm{I}_-$} otherwise.
Figure~\ref{fig:parity-by-layer} shows how these proportions vary across
layers in each model, using the same partition as Table~\ref{tab:head-types}.

\begin{table}[!htbp]
\centering
\small
\begin{tabular}{lrrrr}
\toprule
model & heads & Type~$\mathrm{I}_+$ & Type~$\mathrm{II}$ & Type~$\mathrm{I}_-$ \\
\midrule
\texttt{distilgpt2} & $72$ & $34.7$ & $26.4$ & $38.9$ \\
\texttt{gpt2} & $144$ & $36.8$ & $34.0$ & $29.2$ \\
\texttt{gpt2-medium} & $384$ & $54.2$ & $28.6$ & $17.2$ \\
\texttt{gpt2-large} & $720$ & $64.3$ & $22.6$ & $13.1$ \\
\texttt{opt-125m} & $144$ & $79.2$ & $16.7$ & $4.2$ \\
\texttt{opt-350m} & $384$ & $78.9$ & $20.8$ & $0.3$ \\
\texttt{opt-1.3b} & $768$ & $72.3$ & $17.6$ & $10.2$ \\
\texttt{opt-2.7b} & $1{,}024$ & $71.4$ & $14.8$ & $13.8$ \\
\texttt{opt-6.7b} & $1{,}024$ & $78.7$ & $15.9$ & $5.4$ \\
\texttt{opt-13b} & $1{,}600$ & $75.7$ & $6.0$ & $18.3$ \\
\texttt{bloom-3b} & $960$ & $29.7$ & $49.9$ & $20.4$ \\
\texttt{bloom-7b1} & $960$ & $23.2$ & $55.2$ & $21.6$ \\
\texttt{pythia-70m-deduped} & $48$ & $18.8$ & $50.0$ & $31.2$ \\
\texttt{Mistral-Small-24B-Base} & $1{,}280$ & $36.8$ & $53.8$ & $9.5$ \\
\midrule
all & $9{,}512$ & $57.4$ & $28.5$ & $14.1$ \\
\bottomrule
\end{tabular}
\caption{Percentage of the heads by parity, classified using $\pi(S)$, with tolerance $c_S+d_S\leq\tfrac1{50}$ for Type~$\mathrm{II}$.}
\label{tab:head-types}
\end{table}

\subsection{Eigenvalue distributions of trained heads}
\label{subsec:spectral-clusters}

Whereas earlier experiments used coarse measures on eigenvalues,
in this section we attempt to model the distribution itself.
Darveshi~\cite{darveshi2025spectral} observed that the absolute values of
the eigenvalue distributions of
the bilinear forms in the \texttt{Qwen3} family are well-modeled by
a gamma distribution.
We corroborate their findings for the fourteen models in our survey.
We also examined the symmetric and antisymmetric parts separately.
For the symmetric parts, the distribution is a bimodal
sum of gamma-like distributions
(Observation~\ref{obs:bimodal-gamma}).
For the absolute-values and the antisymmetric parts,
we made the following observation.

\begin{observation}\label{obs:full-gamma}
The absolute values of the nonzero
eigenvalues of the bilinear forms, and the positive imaginary parts
$\sigma(T)$ of the eigenvalues of their antisymmetric parts, are
unimodal and well-modeled by a gamma distribution.
\end{observation}

\subsubsection{Gamma vs.~lognormal vs.~Weibull}

We compared the gamma distribution's fit with two other families.
For each head we fitted a gamma, a lognormal
and a Weibull distribution to each of the following quantities: the
magnitudes $|\lambda(L)|$ of the nonzero eigenvalues of the full form,
the two lobes $\lambda_+(S)$ and $\lambda_-(S)$ of the spectrum of the
symmetric part, and the positive imaginary parts $\sigma(T)$ of the
eigenvalues of the antisymmetric part.
Figure~\ref{fig:goodness-of-fit} reports, for each fit, the average
error between the observed histogram and the fitted density.

\begin{figure}[!htbp]
\centering
\includegraphics[width=0.7\textwidth]{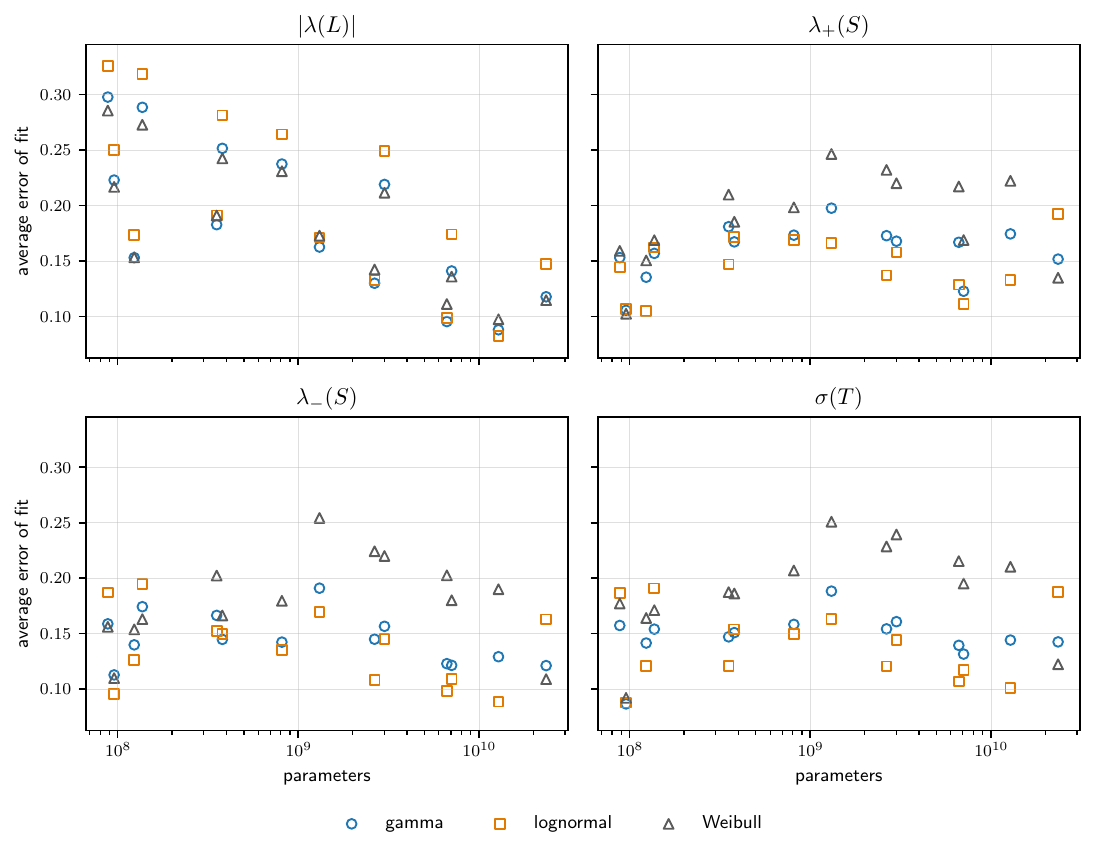}
\caption{Average error of each fitted family over the heads of a model,
against the size of the model; lower is a closer fit.  The error is the
root mean square difference between the heights of a $20$-bin histogram
of the sample and the fitted density, after rescaling the sample to unit mean.}
\label{fig:goodness-of-fit}
\end{figure}

The three families fit about equally well.  Over all heads, the
median error is between $0.12$ and $0.21$ for every family and every
quantity.
The lognormal is closest on the two lobes and on $\sigma(T)$
in a majority of models, while on $|\lambda(L)|$ the
gamma and the Weibull are closest about equally often.

\subsubsection{Relations between \texorpdfstring{$\lambda_\pm$ and $\sigma$}{lambda+/- and sigma}}

We observed that the eigenvalue distributions of $S$ and $T$ are closely related.
To explain this we use the gamma fit in the observations below
because these relations can be naturally described using
its shape and mode parameters.
Let $k_T$ and $\tau$
denote the shape and mode of $\sigma(T)$, and $k_\pm$ and $\mu, \nu$
the shapes and the negative and positive modes of the two lobes 
of the eigenvalue distribution of $S$.
For the comparisons of modes and shapes, we retain only heads for
which all three fitted distributions have positive modes.
(This excludes $550$ of the $9{,}512$ bilinear forms because at least
one fitted distribution has mode zero, leaving $8{,}962$ forms.)

\begin{observation}\label{obs:modes}
The fitted mode of $\sigma(T)$ is close
to the geometric mean of the two lobe modes of $S$: the median of
$\tau / \sqrt{|\mu| \nu}$ over the retained heads of each model
lies between $1.01$ and $1.05$.
\end{observation}

It is interesting to note this relation holds exactly for rank-one forms. 
If $L = h\,uv^T$ for $h>0$ with unit vectors $u, v$ and $t = u^Tv$, 
then $S$ has the nonzero eigenvalues $h(1 \pm t)/2$ and 
$T$ has the single repeated singular value $h\sqrt{1-t^2}/2$, 
which is their geometric mean.
Nonetheless trained bilinear forms are 
far from rank one (cf.~\S\ref{sec:effectiveRank}).

\begin{observation}\label{obs:shapes}
The fitted shape of $\sigma(T)$ is close
to the geometric mean of the two lobe shapes of $S$, with a modest
upward bias: the median of $k_T / \sqrt{k_+ k_-}$ over the retained
heads of each model lies between $1.02$ and $1.15$.
\end{observation}

\begin{observation}\label{obs:sigma-lambda}
The singular values of $T$ and the
absolute values of the eigenvalues of $S$ are strongly correlated
within heads.  Sorting both lists in decreasing order, the median
of their correlation over all heads of each model lies between
$0.986$ and $0.997$.
\end{observation}

\section{Why heads are balanced}\label{sec:balanced}

The most striking observation from the eigenvalue experiments
is that \emph{the symmetric part of every attention head has equally many positive and negative eigenvalues}
(Observation~\ref{obs:balance}).
This is explained by the following theorem.
Let $H$ be a real vector space of dimension $n$
equipped with an inner product $\langle \cdot,\cdot \rangle$.
For linear maps $Q, K \colon C \to H$,
let $A\colon C^\ell \to C^\ell$ denote the length $\ell$ attention head on $C$
constructed from $Q$ and $K$ as query and key maps.
Call the head $A$ \defn{nondegenerate}
if the stacked map
\[
    \binom{Q}{K} \colon C \to H\oplus H,
    \qquad
    x \mapsto (Qx, Kx)
\]
is surjective.

\begin{theorem}
\label{thm:heads-have-balanced-attention}
Let $A$ be a head in an attention block with query and key maps
$Q, K$.
If $A$ is nondegenerate,
then the symmetric bilinear form
\[
    S(x,y)
    = \tfrac12\bigl(\langle Kx, Qy\rangle
    + \langle Qx, Ky\rangle\bigr)
\]
has signature
\[
    (n_+, n_-, n_0) = (n,\, n,\, N - 2n).
\]
\end{theorem}

To prove the theorem we start with an observation
about symmetric parts of low-rank matrices.

\begin{lemma}\label{lem:rank-balance}
Let $L \in M_N(\RR)$ with $\operatorname{rank} L \leq n$, and let
$(n_+, n_-, n_0)$ be the signature of its symmetric part
$S = \tfrac12(L + L^T)$.
Then $n_+ \leq n$ and $n_- \leq n$.
\end{lemma}

\begin{proof}
Let $K = \ker L$, of dimension at least $N - n$.
For $x \in K$,
$2\, x^T S x = x^T L x + x^T L^T x = 0$, so the quadratic form of
$S$ vanishes on $K$.
Let $V_+$ be the span of the eigenvectors of $S$ with positive
eigenvalues, of dimension $n_+$.
Then $V_+ \cap K = 0$: a nonzero vector $x$ of the intersection
would satisfy both $x^T S x > 0$ and $x^T S x = 0$.
Hence $n_+ + (N - n) \leq N$, that is, $n_+ \leq n$; the same
argument with the negative eigenvectors gives $n_- \leq n$.
\end{proof}

\begin{remark}
Low rank alone does not force balance: a positive
semidefinite symmetric $L$ of rank $n$ has $S = L$ of signature
$(n, 0, N - n)$.
The nondegeneracy hypothesis of
Theorem~\ref{thm:heads-have-balanced-attention} is what forces
$S$ to have rank $2n$.
\end{remark}

\begin{proof}[Proof of Theorem~\ref{thm:heads-have-balanced-attention}]
The Gram matrix of $2S$ is $K^T Q + Q^T K$, the symmetric
part of $2\,K^T Q$; since $K^T Q$ factors through $H$, it
has rank at most $n$.
By Lemma~\ref{lem:rank-balance} it therefore suffices to
show that $S$ has rank $2n$.
Let $E = H \oplus H$ with the symmetric bilinear form
$B\bigl((a,b),(a',b')\bigr) = \langle a, b'\rangle
+ \langle b, a'\rangle$, which is nondegenerate, and let
$J \colon C \to E$ be the stacked map, so that
$B(Jx, Jy) = 2 S(x, y)$.
If $Jx = 0$ then $S(x, \cdot) = 0$; conversely, if
$S(x, \cdot) = 0$ then $B(Jx, \cdot)$ vanishes on $JC = E$, so
$Jx = 0$ by nondegeneracy of $B$.
The radical of $S$ is therefore $\ker J$, of dimension $N - 2n$ by
surjectivity, and $S$ has rank $2n$.
\end{proof}

\subsubsection{Reconciling Migliarini's count with Theorem~\ref{thm:heads-have-balanced-attention}}
Migliarini~\cite{migliarini2025selfhating}
studies one of the heads in \texttt{gpt2},
reaching a conclusion
which is apparently in conflict with our Theorem~\ref{thm:heads-have-balanced-attention}.
The matrix Migliarini
calls $\mathtt{W\_QK}$ is the operator
\[
    W_{QK} := Q^T\,K \in \End(C).
\]
Its symmetrization, which Migliarini calls
$\mathtt{W\_sym}$, is the matrix
\[
    W_{\mathrm{sym}}
    := \tfrac12\bigl(W_{QK}+W_{QK}^T\bigr).
\]
Migliarini reports that ``$33$ of $64$ eigenvalues'' of
$W_{\mathrm{sym}}$ are negative.
In fact Theorem~\ref{thm:heads-have-balanced-attention}
implies that $W_{\mathrm{sym}}$ has
precisely $64$ positive and $64$ negative eigenvalues.
Migliarini's count is actually the negative count
among the subset of the largest $64$ eigenvalues ordered by magnitude.

\section{Trained bilinear forms in attention heads}\label{sec:mechinterp}

In this section we attempt to understand what types of bilinear forms
arise in trained language models.
A different interpretation based on moments
is described in Appendix~\ref{app:moments}.
We introduce a map on real bilinear forms valued in the $3$-simplex:
\[
    \pi \colon M_N(\RR) \smallsetminus \{0\} \longrightarrow \Delta_3,
    \qquad
    L \longmapsto (a, b, c, d).
\]
The map $\pi$ is invariant under positive scaling and orthogonal
congruence $L \mapsto Q^T L Q$.

\subsection{The profile of a real bilinear form}
\label{subsec:simplex}

Let $L \in M_N(\RR)$ with $\|L\| = 1$, and write $L = S + T$ with
$S$ symmetric and $T$ antisymmetric.
Let $\lambda_+$ (resp.\ $\lambda_-$) list the positive eigenvalues (resp.\
the absolute values of the negative eigenvalues) of $S$ in
decreasing order, appending zeros so that the two lists have equal
length.
For a real vector $x \in \RR^m$
let $x_+ = \max(x,0)$ and $x_- = \max(-x,0)$
applied entry-wise.

\begin{definition}\label{defn:profileMap}
The \defn{profile} of $L$ is $\pi(L)=(a,b,c,d)$, where
\[
    a = \|T\|^2,
    \qquad
    b = 2\,\langle \lambda_+, \lambda_- \rangle,
    \qquad
    c = \|(\lambda_+ - \lambda_-)_+\|^2,
    \qquad
    d = \|(\lambda_+ - \lambda_-)_-\|^2.
\]
For nonzero $L \in M_N(\RR)$ we set $\pi(L) = \pi(L/\|L\|)$.
\end{definition}

The components of the profile are nonnegative and satisfy
$a + b + c + d = 1$.
The facet $\{a = 0\}$ is the triangle
    $\Delta = \{(b, c, d) \in \RR^2 : b, c, d \geq 0,\ b + c + d = 1\}$
containing the profiles of symmetric forms.

\begin{theorem}\label{thm:simplex-faces}\,
\begin{enumerate}
    \item $a = 0$ if and only if $L$ is symmetric, and $a = 1$ if
    and only if $L$ is antisymmetric;
    \item $b = 0$ if and only if $S$ is semidefinite, and $b = 1$
    if and only if $L$ is symmetric with $\lambda_+ = \lambda_-$;
    \item $d = 0$ if and only if $S = H + P$ for symmetric matrices
    $H$ and $P$ such that the spectrum of $H$ is symmetric about
    zero and $P$ is positive semidefinite; equivalently,
    $(\lambda_+)_j \geq (\lambda_-)_j$ for every $j$.
    In that case $H$ and $P$ may be chosen to commute with $S$ and
    with each other, with $\|P\|^2 = c\,\|L\|^2$.
    Moreover, $d = 1$ if and only if $L$ is symmetric negative
    semidefinite;
    \item $c = 0$ if and only if $S = H - P$ with $H$ as in (3) and
    $P$ positive semidefinite; equivalently,
    $(\lambda_-)_j \geq (\lambda_+)_j$ for every $j$; and then
    $\|P\|^2 = d\,\|L\|^2$.
    Moreover, $c = 1$ if and only if $L$ is symmetric positive
    semidefinite.
\end{enumerate}
\end{theorem}

\begin{proof}
See Appendix~\ref{app:bimodal}.
\end{proof}

\begin{figure}[!htbp]
\centering
\includegraphics[width=\textwidth]{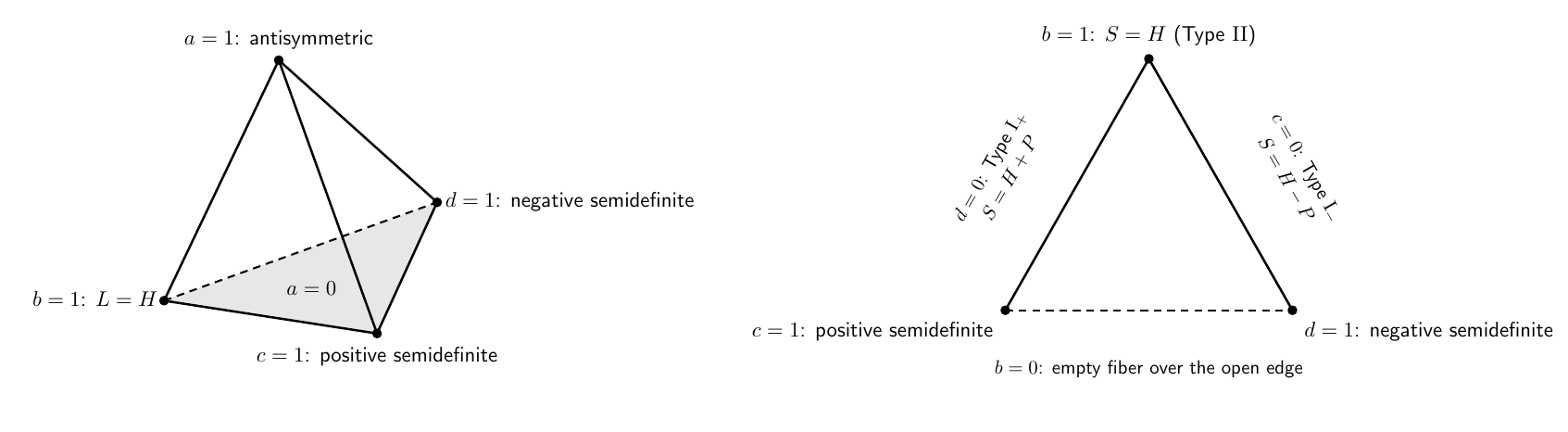}
\caption{The simplex $\Delta_3$ of profiles of
bilinear forms (left) and its face $\Delta$ of symmetric forms
(right). Here $H$ has spectrum symmetric about zero and $P$ is positive
semidefinite. The fiber over the dashed edge
$\{b = 0\}$ is empty away from its endpoints.}
\label{fig:simplex-faces}
\end{figure}

\begin{table}[!htbp]
\centering
\begin{tabular}{lll}
\toprule
& face & fiber of $\pi$ \\
\midrule
vertex & $a = 1$ & antisymmetric forms \\
vertex & $b = 1$ & symmetric forms with $\lambda_+ = \lambda_-$ \\
vertex & $c = 1$ & positive semidefinite forms \\
vertex & $d = 1$ & negative semidefinite forms \\
edge & $\{c = d = 0\}$ & forms with $\lambda_+ = \lambda_-$ \\
edge & $\{b = d = 0\}$ & forms with $S$ positive semidefinite \\
edge & $\{b = c = 0\}$ & forms with $S$ negative semidefinite \\
edge & $\{a = d = 0\}$ & symmetric forms $H + P$ \\
edge & $\{a = c = 0\}$ & symmetric forms $H - P$ \\
edge & $\{a = b = 0\}$ & empty away from the endpoints \\
facet & $\Delta=\{a = 0\}$ & symmetric forms \\
facet & $\{b = 0\}$ & forms with $S$ semidefinite;
relative interior empty \\
facet & $\{d = 0\}$ & forms with $S = H + P$ \\
facet & $\{c = 0\}$ & forms with $S = H - P$ \\
interior & & $T \neq 0$ and $S$ of neither form
$H + P$ nor $H - P$ \\
\bottomrule
\end{tabular}
\caption{The fibers of $\pi$ over the faces of
$\Delta_3$; $H$ denotes a symmetric matrix with spectrum symmetric
about zero and $P$ a positive semidefinite matrix.}
\label{tab:simplex-fibers}
\end{table}

\begin{lemma}\label{lem:profile-scale}
Let $L\in M_N(\RR)$ with symmetric part $S \neq 0$, and write
$\pi(L) = (a, b, c, d)$ and $\pi(S) = (0, b_S, c_S, d_S)$.
Then
    $(b, c, d) = (1 - a)\,(b_S, c_S, d_S)$.
\end{lemma}

\begin{proof}
The three quantities $b, c, d$ are computed from the eigenvalues
of $S$ and divided by $\|L\|^2$, whereas $b_S, c_S, d_S$ are the
same quantities divided by $\|S\|^2$.
The claim follows from
$\|S\|^2 / \|L\|^2 = 1 - \|T\|^2/\|L\|^2 = 1 - a$.
\end{proof}

\begin{proposition}\label{prop:rankOneLocus}
Let $L$ be a rank-one matrix of unit norm.
Then
$$
\pi(L) = \left( \frac{1-t^2}{2}, \frac{1-t^2}{2}, t^2_+, t^2_- \right),
\qquad \text{where } t = \tr(L).
$$
The set of profiles of rank-one matrices is equal to
$\Theta = \{a = b, \,cd = 0\} \subset \Delta_3$.
\end{proposition}

\begin{proof}
Write $L = uv^T$ for unit vectors $u,v$.
If $u$ and $v$ are orthogonal then $L$ is nilpotent and
has no nonzero eigenvalues,
and otherwise the eigenspace of the nonzero eigenvalue is spanned by $u$
and the nonzero eigenvalue is $\langle u,v \rangle$.
Thus $t = \langle u,v \rangle$.

We bring $u$ to the first standard basis vector
by an orthogonal change of coordinates.
If $u = \pm v$ then the top-left entry of $L$ is $\pm1$ with all other entries zero.
If $u = v$ then $\lambda_+ = (1,0,\ldots,0)$, $\lambda_- = 0$, and
$\pi(L) = (0,0,1,0)$.
If $u= -v$ then $\lambda_+ = 0$, $\lambda_- = (1,0,\ldots,0)$, and
$\pi(L) = (0,0,0,1)$.
This verifies the formula when $t = \pm 1$.

Suppose $|t|<1$.
Define the vectors
$$
e = \frac{u+v}{\sqrt{2+2t}},
\qquad
f = \frac{v-u}{\sqrt{2-2t}}.
$$
These form an orthonormal set.
We may extend this to an orthonormal basis in which $L$
has top-left block equal to
$$
\frac12
\begin{pmatrix}
1+t & \sqrt{1-t^2}   \\
-\sqrt{1-t^2} & t-1  \\
\end{pmatrix}
$$
and elsewhere is zero.
This block is equal to
$$
S+T=
\frac12
\begin{pmatrix}
1+t & 0   \\
0 & t-1  \\
\end{pmatrix}
+
\frac12
\begin{pmatrix}
0& \sqrt{1-t^2}   \\
-\sqrt{1-t^2} & 0 \\
\end{pmatrix}.
$$
From this we find $\tr(T^TT)=\tfrac12(1-t^2)$,
$\lambda_+ = \frac12(1+t,0)$,
$\lambda_- = \frac12(1-t,0)$,
and the rest follows easily.
\end{proof}

\subsection{Profiles of trained attention heads}
\label{subsec:head-simplex}

In this section we study the attention heads of the fourteen models
with the profile map.
To attempt a partial explanation for Observation~\ref{obs:head-profiles},
we consider an ensemble of large symmetric matrices
whose positive and negative eigenvalue distributions have the same
shape up to rescaling.

\begin{theorem}\label{thm:accumulation}
Let $X$ be a positive random variable with finite,
nonzero second moment.
Fix $u,v>0$ and set $\rho=u/v$.
For each $n$ let $N_n \geq 2n$ and let $S_n \in M_{N_n}(\RR)$ be a
symmetric matrix with positive eigenvalues $vX_1,\ldots,vX_n$
and negative eigenvalues $-uX'_1,\ldots,-uX'_n$, where the $X_i$ and
$X'_i$ are independent copies of $X$.
Then the profile $\pi(S_n)$ converges almost surely to
\begin{equation}\label{eqn:limitProfiles}
    \frac{1}{1 + \rho^2}
    \bigl( 0,\; 2\rho,\; (1 - \rho)_+^2,\; (\rho - 1)_+^2 \bigr) \in p(\Theta).
\end{equation}
As $\rho$ ranges over $(0, 1]$ the limit traces the edge
$\{d = 0\}$ of $\Delta$ from $(0,1,0)$ to $(1,0,0)$, and as $\rho$
ranges over $[1, \infty)$ it traces the edge $\{c = 0\}$ from
$(1,0,0)$ to $(0,0,1)$.
\end{theorem}

For exactly proportional lobes, $\lambda_- = \rho\,\lambda_+$, direct
substitution in the profile definition gives
\eqref{eqn:limitProfiles}.
The proof consists of showing that for randomly sampled matrices,
this formula still holds in the high-dimensional limit.

\begin{proof}
Let $\alpha_n$ and $\beta_n$ denote the positive eigenvalues of
$S_n$ and the absolute values of its negative eigenvalues, each
listed in decreasing order, so that $\lambda_+ = \alpha_n$ and
$\lambda_- = \beta_n$.
By Lemma~\ref{lem:proportional-lobes} the sorted lobes become
proportional in the limit:
$\|\beta_n - \rho\, \alpha_n\| / \|\alpha_n\| \to 0$ almost surely.
Write $\beta_n = \rho\, \alpha_n + \varepsilon_n$ with
$\delta_n = \|\varepsilon_n\| / \|\alpha_n\| \to 0$.
By the Cauchy--Schwarz and triangle inequalities,
\[
    \Bigl|
        \frac{\langle \alpha_n, \beta_n \rangle}{\|\alpha_n\|^2}
        - \rho
    \Bigr|
    \leq \delta_n ,
    \qquad
    \Bigl|
        \frac{\|\beta_n\|^2}{\|\alpha_n\|^2}
        - \rho^2
    \Bigr|
    \leq \delta_n^2 + 2\rho\,\delta_n ,
\]
and, since $x \mapsto x_\pm$ is $1$-Lipschitz in each entry and
$\bigl( (1-\rho)\,\alpha_n \bigr)_\pm = (1-\rho)_\pm\, \alpha_n$,
\[
    \Bigl|
        \frac{\|(\alpha_n - \beta_n)_\pm\|}{\|\alpha_n\|}
        - (1 - \rho)_\pm
    \Bigr|
    \leq
    \frac{\|(\alpha_n - \beta_n) - (1-\rho)\,\alpha_n\|}{\|\alpha_n\|}
    = \delta_n .
\]
Since $\|S_n\|^2 = \|\alpha_n\|^2 + \|\beta_n\|^2$, dividing each of
$\|S_n\|^2$, $2\langle \alpha_n, \beta_n\rangle$ and
$\|(\alpha_n - \beta_n)_\pm\|^2$ by $\|\alpha_n\|^2$ gives quantities
that converge almost surely to $1 + \rho^2$, $2\rho$ and
$(1-\rho)_\pm^2$ respectively.
The profile map is homogeneous of degree zero in $S_n$, so
\[
    b
    = \frac{2\langle \alpha_n, \beta_n\rangle}{\|S_n\|^2}
    = \frac{2\langle \alpha_n, \beta_n\rangle / \|\alpha_n\|^2}
           {\|S_n\|^2 / \|\alpha_n\|^2} ,
\]
and likewise for $c$ and $d$; the factor $\|\alpha_n\|^2$ cancels,
and each ratio converges almost surely to the corresponding entry of
\eqref{eqn:limitProfiles}.
\end{proof}

\begin{figure}[!htbp]
\centering
\includegraphics[width=0.6\textwidth]{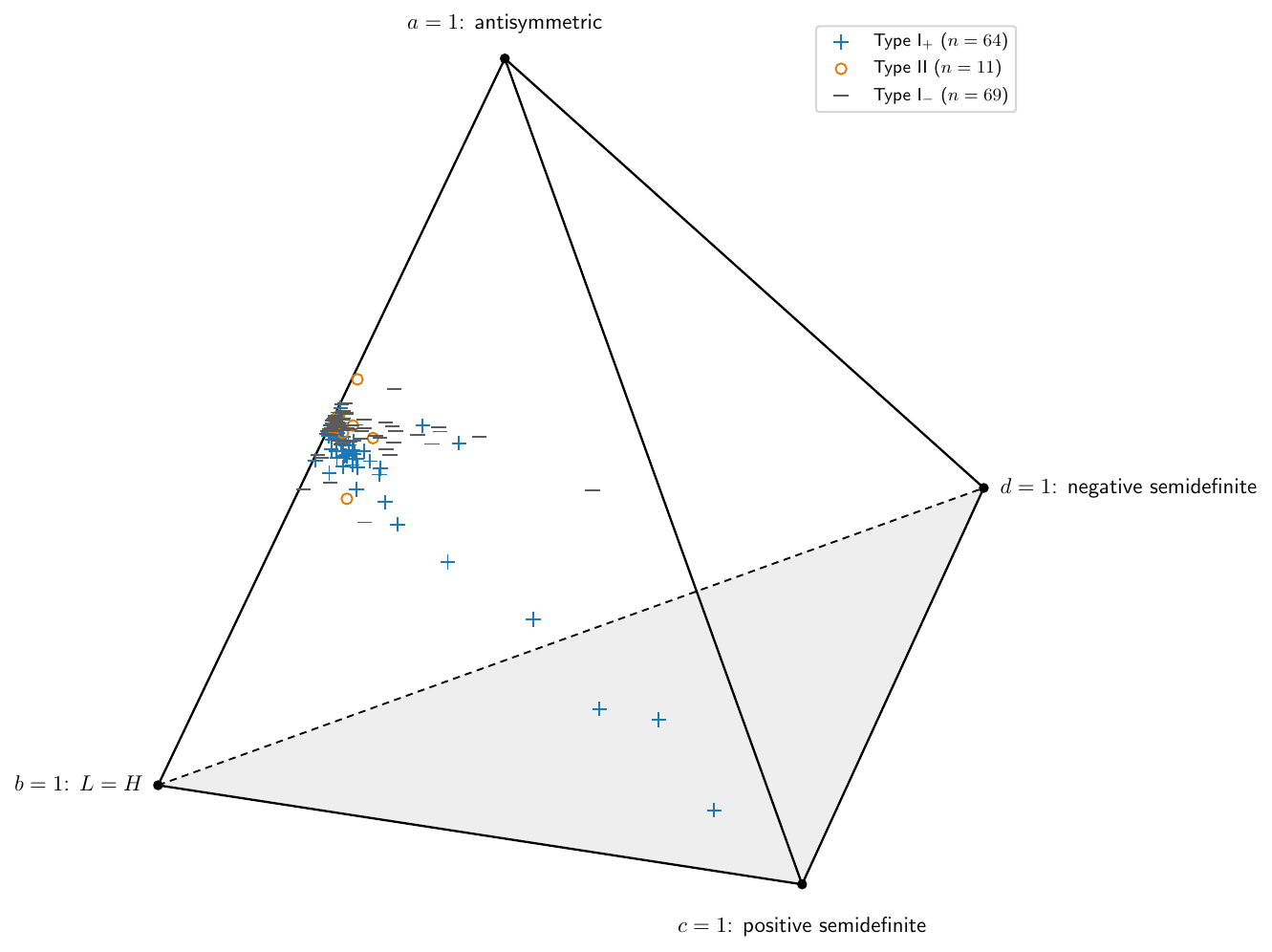}
\caption{Profiles of the heads of \texttt{gpt2} on the simplex
$\Delta_3$.}
\label{fig:abcd-simplex-scatter}
\end{figure}

\begin{figure}[!htbp]
\centering
\includegraphics[width=0.5\textwidth]{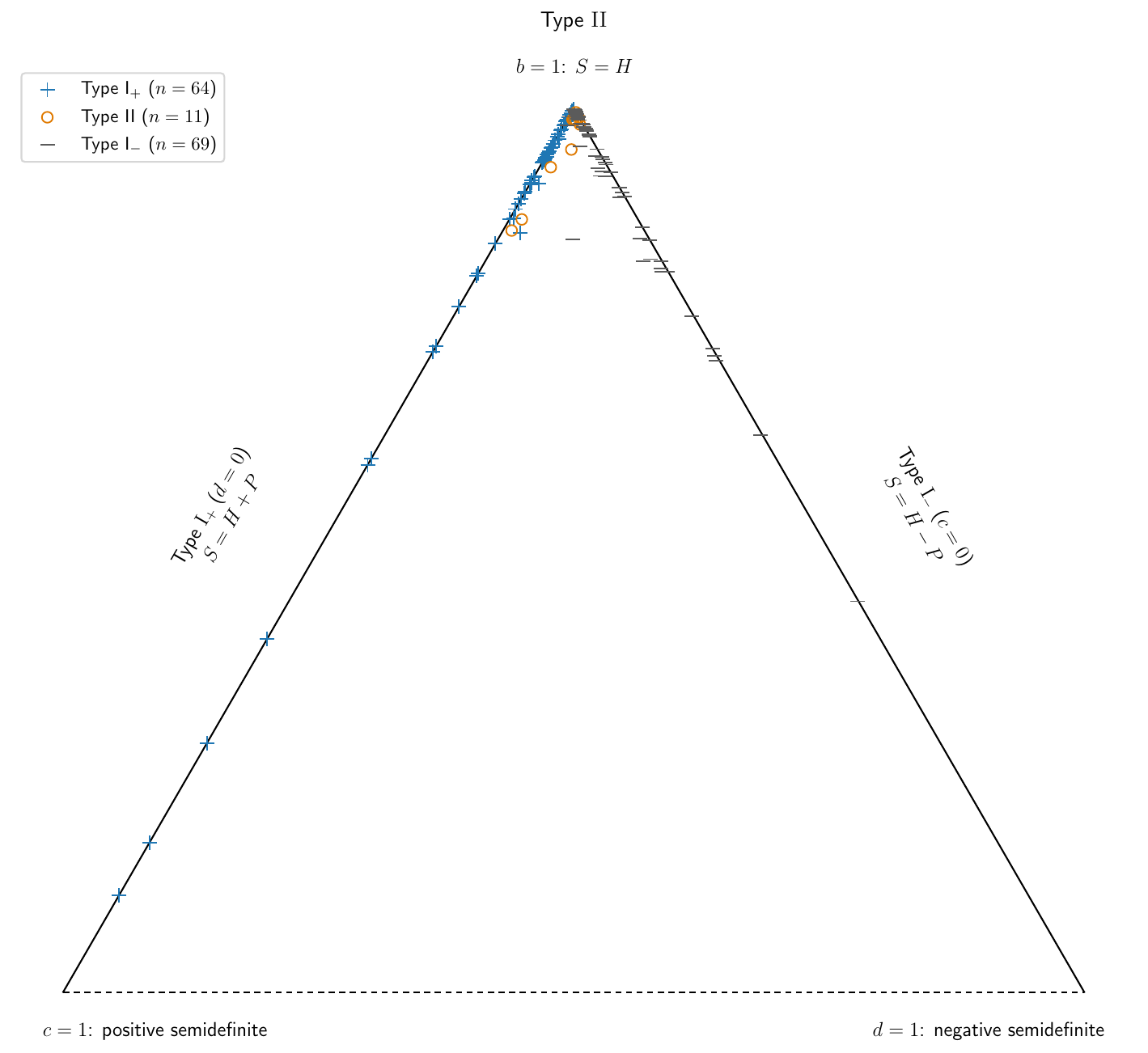}
\caption{Profiles of the symmetric parts of the heads of \texttt{gpt2} on $\Delta$.
Here $H$ has spectrum symmetric about zero and $P$ is positive semidefinite.}
\label{fig:bcd-simplex-scatter}
\end{figure}

\subsubsection{The hypothesis of Theorem~\ref{thm:accumulation}}\label{sec:proportionalityHypothesis}

Observation~\ref{obs:pairing} quantifies the approximate proportionality
of the two sorted spectral lobes at trained heads.
Figure~\ref{fig:accumulation-example} illustrates the theorem at
$\rho = \tfrac12$, where the limiting profiles of $S_n$ are
$(b, c, d) = (0.8, 0.2, 0)$.

\begin{figure}[!htbp]
\centering
\includegraphics[width=0.9\textwidth]{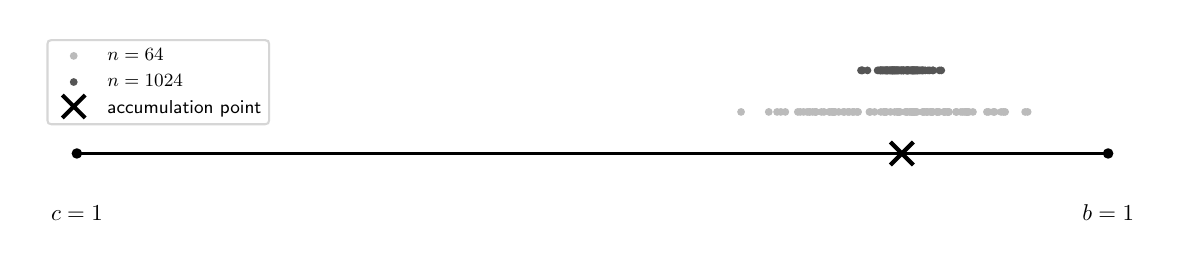}
\caption{Profiles of $100$ spectra sampled from
$\tfrac12(\gamma_{\mu,k} + \gamma_{\nu,k})$ with $\mu = -\tfrac12$,
$\nu = 1$, $k = 2$, for $n = 64$ and $n = 1024$ eigenvalues per
side, plotted by their coordinate $b$ along the edge $\{d = 0\}$,
drawn horizontally from the vertex $c = 1$ to the vertex $b = 1$,
together with the limit point $b = 0.8$ of
Theorem~\ref{thm:accumulation} ($\times$).}
\label{fig:accumulation-example}
\end{figure}

\subsubsection{Effective rank of trained bilinear forms}\label{sec:effectiveRank}

Let $\sigma_1 \geq \sigma_2 \geq \cdots$ be the singular values of $L$. 
There are many measures of effective rank used in the literature. 
One of them is the \emph{participation ratio} defined by 
$(\sum_i \sigma_i^2)^2 / \sum_i \sigma_i^4$. 
This equals $1$ for a rank-one form and $n$ for a flat spectrum. 
For the heads of the fourteen models, 
the participation ratio has median $48.5$ 
and lies below $2$ for fewer than $1\%$ of them.

\section{Discussion}\label{sec:discussion}

We have presented several observations on trained attention heads
showing that their bilinear forms have distinguished properties.
It is important to discern which properties are learned and which are structural.
Equal counts for positive and negative eigenvalues is
a structural consequence of
the low-rank query-key attention head construction
rather than a pattern learned from language.
Training shapes the distributions of these eigenvalues: across the
models studied, the symmetric part is stronger than a random baseline,
and its positive and negative eigenvalues follow closely related
distributions.
We have proposed metrics using the classification of real bilinear forms
which describe variation between trained heads.
They are useful metrics for head classification,
although they do not by themselves
identify what a head computes in an explicit mechanistic sense.

Several observations remain unexplained.  In particular, we do not know
why the eigenvalues of the symmetric part and the singular values of the
antisymmetric part have gamma-like profiles, or why their fitted
parameters and magnitudes are so closely related.  These patterns occur
across all fourteen models studied, but the present results do not determine
whether they arise from the parametrization, optimization, training data,
or an interaction among them.
The curious observation that profiles of trained bilinear forms accumulate near
the family of profiles of rank-one forms also remains only partially explained.

The empirical scope is also limited.  We examine the bilinear forms in
the heads of fourteen decoder-only models.
However an attention head acts within a larger system:
its output is combined with other heads and then processed by the neural
network in each transformer block.  Future work should examine how the
symmetric and antisymmetric parts and their spectral distributions
interact with these components.

\begin{table}[!htbp]
\centering
\small
\setlength{\tabcolsep}{12pt}
\begin{tabular}{@{}p{0.4\linewidth} p{0.35\linewidth} p{0.07\linewidth}@{}}
\toprule
Observation & Status & Source \\
\midrule
\ref{obs:balance}:
the symmetric part is balanced
& Proved: Theorem~\ref{thm:heads-have-balanced-attention}
&  \\[1ex]

\ref{obs:head-profiles}:
head profiles accumulate near rank-one profiles $\Theta \subset \Delta_3$
& Partially explained: Theorem~\ref{thm:accumulation}; the
proportionality of the lobes in its proof holds approximately at
the heads (\S\ref{subsec:head-simplex})
& \\[1ex]

\ref{obs:bimodal-gamma}:
$\lambda(S)$ has two gamma-like lobes
& Unexplained
&  \\[1ex]

\ref{obs:symmetric-bias}:
symmetric bias of attention heads
& Partially explained: \cite[Theorem~2.4]{saponati2025underlying}
for the bidirectional case with assumptions, unexplained for decoder-only
& \cite{saponati2025underlying} \\[1ex]

\ref{obs:pairing}:
the sorted spectral lobes are approximately proportional
& Linked to profile accumulation by Theorem~\ref{thm:accumulation};
see \S\ref{sec:proportionalityHypothesis}
&  \\[1ex]

\ref{obs:full-gamma}:
$|\lambda(L)|$ and $\sigma(T)$ have gamma-like profiles
& Unexplained
& \cite{darveshi2025spectral} \\[1ex]

\ref{obs:modes}:
$\tau \approx \sqrt{|\mu| \nu}$
& Unexplained
&  \\[1ex]

\ref{obs:shapes}:
$k_T \approx \sqrt{k_+ k_-}$
& Unexplained
&  \\[1ex]

\ref{obs:sigma-lambda}:
sorted $\sigma(T)$ and $|\lambda(S)|$ are strongly correlated
& Unexplained
&  \\
\bottomrule
\end{tabular}
\end{table}

\section*{Acknowledgements}

The ideas underlying
Theorem~\ref{thm:heads-have-balanced-attention}
and Proposition~\ref{prop:qkv-homogenizes-to-bilinear}
arose out of extended conversations with ChatGPT 5.5.
The observation that $\Theta$ is the set of profiles of rank-one forms
was made by ChatGPT 6.
Theorem~\ref{thm:accumulation} and its proof,
as well as the constructions in Appendix~\ref{app:moments},
arose out of extended conversations with Claude Fable 5.
The appendices (with the exception of Appendix~\ref{app:moments}),
figures, and experimental code,
as well as the statements and proofs of
Proposition~\ref{prop:near-identity}
and Lemma~\ref{lemma:pairing-bound}
were generated by Claude Fable 5.
The mathematical assertions in the appendices
were formalized in Lean with the help of Claude Fable 5.
The author is responsible for the final text and released code.

\section*{Supplementary information}

All experimental code, Lean formalizations, as well as additional data visualizations are publicly available.\footnote{\url{https://aodesky.github.io/attention-bilinear-forms/}}

\appendix

\section{Random matrices}\label{app:random-matrices}

This appendix collects the symmetric/skew energy identities and the
expected energy splits for random matrices used in
\S\ref{sec:experimentsQK} to justify the random-baseline value
$\mathscr{S}\approx 1/2$ in Figure~\ref{fig:qk-multimodel}.
\begin{lemma}[Symmetric and skew energy identities; cf.~{\cite[Lemma~S1.14 and Eq.~(S118)]{saponati2025underlying}}]
\label{lemma:symmetric-skew-energy-baseline}
Let $L\in M_N(\RR)$, and write
$S=\tfrac12(L+L^T)$, $T=\tfrac12(L-L^T)$. Then
$\|L\|^2=\|S\|^2+\|T\|^2$
and
    $\|S\|^2
    =
    \tfrac12\|L\|^2+\tfrac12\operatorname{tr}(L^2)$,
    $\|T\|^2
    =
    \tfrac12\|L\|^2-\tfrac12\operatorname{tr}(L^2)$.
Consequently,
\[
    \frac{\|S\|^2}{\|L\|^2}
    =
    \frac12+\frac{\operatorname{tr}(L^2)}{2\|L\|^2},
    \qquad
    \frac{\|T\|^2}{\|L\|^2}
    =
    \frac12-\frac{\operatorname{tr}(L^2)}{2\|L\|^2}.
\]
\end{lemma}

\begin{proof}
The subspaces of symmetric and antisymmetric matrices are orthogonal with
respect to the Frobenius inner product $\langle X,Y\rangle
=\operatorname{tr}(XY^T)$: for $S$ symmetric and $T$ antisymmetric,
\[
    \langle S,T\rangle
    =
    \operatorname{tr}(ST^T)
    =
    -\operatorname{tr}(ST)
    =
    -\operatorname{tr}\bigl((ST)^T\bigr)
    =
    -\operatorname{tr}(T^TS^T)
    =
    \operatorname{tr}(TS)
    =
    \operatorname{tr}(ST),
\]
so $\operatorname{tr}(ST)=0$ and $\langle S,T\rangle=0$.
Since $L=S+T$, it follows that
$\|L\|^2=\|S\|^2+\|T\|^2$. Expanding,
\[
    \|S\|^2
    =
    \tfrac14\bigl(\langle L,L\rangle+2\langle L,L^T\rangle
        +\langle L^T,L^T\rangle\bigr).
\]
Since $\langle L,L\rangle=\langle L^T,L^T\rangle=\|L\|^2$ and
$\langle L,L^T\rangle=\operatorname{tr}(L^2)$, we obtain
$\|S\|^2=\tfrac12\|L\|^2+\tfrac12\operatorname{tr}(L^2)$. The formula
for $\|T\|^2$ follows by subtraction from $\|L\|^2$.
\end{proof}

\begin{corollary}[Iid square random matrix baseline]
\label{cor:iid-square-symmetric-skew-baseline}
Let $L\in M_N(\RR)$ be a random matrix whose entries are independent,
mean-zero, with common variance $\sigma^2$. Then
\[
    \EE\|S\|^2
    =
    \frac{N(N+1)}{2}\sigma^2,
    \qquad
    \EE\|T\|^2
    =
    \frac{N(N-1)}{2}\sigma^2,
\]
so
\[
    \frac{\EE\|S\|^2}{\EE\|L\|^2}
    =
    \frac{N+1}{2N},
    \qquad
    \frac{\EE\|T\|^2}{\EE\|L\|^2}
    =
    \frac{N-1}{2N}.
\]
If the entries of $L$ are iid Gaussian, then also
$\EE\!\left[\|S\|^2/\|L\|^2\right]=(N+1)/(2N)$ and
$\EE\!\left[\|T\|^2/\|L\|^2\right]=(N-1)/(2N)$.
\end{corollary}

\begin{proof}
For $i<j$, $S_{ij}=S_{ji}=\tfrac12(L_{ij}+L_{ji})$ with $L_{ij},L_{ji}$
independent of variance $\sigma^2$, so $\EE S_{ij}^2=\tfrac12\sigma^2$.
Each off-diagonal pair contributes $\sigma^2$ in expectation; there are
$N(N-1)/2$ such pairs. The diagonal contribution is $\sum_i\EE S_{ii}^2=
\sum_i\EE L_{ii}^2=N\sigma^2$. Hence $\EE\|S\|^2=N\sigma^2+
N(N-1)/2\cdot\sigma^2=N(N+1)\sigma^2/2$. The skew case is similar with
no diagonal contribution.

For the iid Gaussian case, $L$ is an isotropic Gaussian vector in the
$N^2$-dimensional inner-product space $M_N(\RR)$. The symmetric and
antisymmetric subspaces are orthogonal of dimensions $N(N+1)/2$ and
$N(N-1)/2$ respectively, so the expected fraction of squared norm in
each subspace equals its dimension divided by $N^2$.
\end{proof}

\begin{corollary}[Random QK-product baseline]
\label{cor:random-qk-product-symmetric-skew-baseline}
Let $K,Q\in M_{n\times N}(\RR)$ be independent random matrices with
independent, mean-zero entries of variances $\sigma_K^2$ and $\sigma_Q^2$.
Set $L=K^TQ$ and $S,T$ as before. Then
\[
    \EE\|L\|^2=nN^2\sigma_K^2\sigma_Q^2,
    \qquad
    \EE\|S\|^2
    =
    \frac{nN(N+1)}{2}\sigma_K^2\sigma_Q^2,
\]
\[
    \EE\|T\|^2
    =
    \frac{nN(N-1)}{2}\sigma_K^2\sigma_Q^2,
\]
so
\[
    \frac{\EE\|S\|^2}{\EE\|L\|^2}
    =
    \frac{N+1}{2N},
    \qquad
    \frac{\EE\|T\|^2}{\EE\|L\|^2}
    =
    \frac{N-1}{2N}.
\]
\end{corollary}

\begin{proof}
Writing $L_{ab}=\sum_{r=1}^n (K)_{ra}(Q)_{rb}$ and using independence,
$\EE L_{ab}=0$ and $\EE L_{ab}^2=n\sigma_K^2\sigma_Q^2$. Summing over
$a,b\in\{1,\ldots,N\}$ gives $\EE\|L\|^2=nN^2\sigma_K^2\sigma_Q^2$.

By Lemma~\ref{lemma:symmetric-skew-energy-baseline},
$\EE\|S\|^2=\tfrac12\EE\|L\|^2+\tfrac12\EE\operatorname{tr}(L^2)$.
Expanding $\operatorname{tr}(L^2)=\sum_{a,b}L_{ab}L_{ba}$: for $a\neq b$,
$\EE(L_{ab}L_{ba})=0$ by independence and centering; for $a=b$,
$\EE L_{aa}^2=n\sigma_K^2\sigma_Q^2$. Hence
$\EE\operatorname{tr}(L^2)=nN\sigma_K^2\sigma_Q^2$, and
$\EE\|S\|^2=nN(N+1)\sigma_K^2\sigma_Q^2/2$. The skew formula follows by
subtraction.
\end{proof}

The preceding corollary computes the ratio of expectations
$\EE\|S\|^2/\EE\|L\|^2$. Under the additional hypothesis that the
entries are Gaussian, the same value is also the expectation of the
normalized symmetric energy --- and the low-rank head-projection
constraint $\mathrm{rank}(L)\leq n$ does not change this expectation.

\begin{proposition}[Expected symmetric energy for random low-rank QK products]
\label{prop:expected-symmetric-energy-low-rank-qk}
Let $K,Q\in M_{n\times N}(\RR)$ be independent random matrices with iid
centered Gaussian entries, where $1\leq n\leq N$. Set $L=K^TQ\in M_N(\RR)$,
$S=\tfrac12(L+L^T)$, $T=\tfrac12(L-L^T)$, and define
    $\mathscr{S}=\frac{\|S\|^2}{\|L\|^2},
    \mathscr{A}=\frac{\|T\|^2}{\|L\|^2}$.
Then
\[
    \EE[\mathscr{S}]=\frac{N+1}{2N},
    \qquad
    \EE[\mathscr{A}]=\frac{N-1}{2N},
\]
independently of the head dimension $n$.
\end{proposition}

\begin{proof}
By Lemma~\ref{lemma:symmetric-skew-energy-baseline},
    $\mathscr{S}
    =
    \frac12+\frac{\operatorname{tr}(L^2)}{2\|L\|^2}$,
so it suffices to show that
$\EE[\operatorname{tr}(L^2)/\|L\|^2]=1/N$.
Since $K$ and $Q$ have iid Gaussian entries, their joint law is
invariant
under right multiplication of either factor by any $V\in O_N(\RR)$.
Therefore the law of $L=K^TQ$ is invariant under
$L\longmapsto U^TLV$
for every $U,V\in O_N(\RR)$. Set $X=L/\|L\|$, so $\|X\|=1$ almost
surely. By permutation invariance of the row and column indices, the
expectations $\EE[X_{ij}^2]$ are all equal, and since
$\sum_{i,j}X_{ij}^2=1$,
    $\EE[X_{ij}^2]=\frac{1}{N^2}$
for every $i,j$. For $i\neq j$, changing the sign of the $i$-th column of
$K$ flips the $i$-th row of $L=K^TQ$, hence flips $X_{ij}$ while
leaving
$X_{ji}$ unchanged; this preserves the law of $X$, so
$\EE[X_{ij}X_{ji}]=0$. Hence
\[
    \EE\!\left[\frac{\operatorname{tr}(L^2)}{\|L\|^2}\right]
    =
    \EE[\operatorname{tr}(X^2)]
    =
    \sum_{i,j}\EE[X_{ij}X_{ji}]
    =
    \sum_{i=1}^N\EE[X_{ii}^2]
    =
    \frac{1}{N},
\]
and therefore $\EE[\mathscr{S}]=\tfrac12+\tfrac1{2N}=(N+1)/(2N)$.
Since $\mathscr{A}=1-\mathscr{S}$, we get $\EE[\mathscr{A}]=(N-1)/(2N)$.
\end{proof}

\begin{lemma}\label{lem:proportional-lobes}
Let $X$ be a positive random variable with finite, nonzero second
moment.  Fix scales $u,v>0$ and set $\rho=u/v$.
For each $n$, let $\alpha_n\in\RR^n$ be the decreasing
sorting of $vX_1,\ldots,vX_n$ and let $\beta_n\in\RR^n$ be the
decreasing sorting of $uX'_1,\ldots,uX'_n$, where the $X_i$ and $X'_i$
are independent copies of $X$.
Then the sorted lobes become proportional with ratio $\rho$:
\[
    \frac{\|\beta_n - \rho\, \alpha_n\|}{\|\alpha_n\|}
    \longrightarrow 0
    \qquad \text{almost surely as } n \to \infty .
\]
\end{lemma}

\begin{proof}[Proof of Lemma~\ref{lem:proportional-lobes}]
We may write $\beta_n=\rho\,\alpha_n'$, where $\alpha_n'$ is the
decreasing sorting of $vX'_1,\ldots,vX'_n$.  Since the claim is
unchanged when both lists are divided by $v$, we may assume $v=1$.
Let $m_2=\mathbf{E}[X^2]$, which is finite and nonzero by assumption.

Let $\overline F(x) = \mathbf{P}(X > x)$, and let
$G \colon (0,1) \to (0,\infty)$ be the decreasing generalized
inverse
$G(t) = \inf \{ x > 0 : \overline F(x) \leq t \}$, so that
$G(t) > x$ if and only if $t < \overline F(x)$.
Let $U$ be a random variable with $\mathbf{P}(U < t) = t$ for
$t \in [0, 1]$.
Then
$\mathbf{P}(G(U) > x) = \mathbf{P}(U < \overline F(x))
= \overline F(x)$, so $G(U)$ has the law of $X$ and
\[
    \int_0^1 G(t)^2\, dt = m_2 .
\]
Given $n$ independent samples $X_1, \ldots, X_n$ of $X$ with
decreasing sorting $x_1 \geq \cdots \geq x_n$, define the step
function $h_n(t) = x_{\lceil nt \rceil}$ on $(0,1)$, so that
$\int_0^1 h_n^2 = \frac1n \sum_i x_i^2$.
We claim that $h_n \to G$ in $L^2(0,1)$ almost surely, and prove
it in three steps.

First, $h_n(t) \to G(t)$ almost surely for each fixed $t$ at
which $G$ is continuous, hence for all but countably many
$t \in (0,1)$.
Let $\overline F_n(x) = n^{-1} \#\{i : X_i > x\}$, so that
$h_n(t) > x$ exactly when $\overline F_n(x) \geq \lceil nt \rceil
/ n$.
Fix $\varepsilon > 0$.
Since $G$ is continuous at $t$ we have
$\overline F(G(t) + \varepsilon) < t$, so by the strong law of
large numbers
$\overline F_n(G(t) + \varepsilon) \to
\overline F(G(t) + \varepsilon) < t$ almost surely; eventually
$\overline F_n(G(t) + \varepsilon) < t \leq \lceil nt \rceil / n$
and hence $h_n(t) \leq G(t) + \varepsilon$.
The lower bound is symmetric, using
$\overline F(G(t) - \varepsilon) > t$.

Second, by the strong law of large numbers,
\[
    \int_0^1 h_n(t)^2\, dt
    = \frac{1}{n} \sum_{i=1}^n X_i^2
    \longrightarrow m_2
    = \int_0^1 G(t)^2\, dt
    \qquad \text{almost surely.}
\]

Third, expanding
$\int (h_n - G)^2 = \int h_n^2 + \int G^2 - 2 \int h_n G$
and applying Fatou's lemma to the nonnegative functions
$h_n G$, which converge to $G^2$ almost everywhere by the first
step and so give $\liminf \int h_n G \geq \int G^2$, we
conclude
\[
    \limsup_n \int_0^1 (h_n - G)^2\, dt
    \;\leq\; m_2 + m_2 - 2 m_2 = 0,
\]
which proves the claim.

Now let $g_n$ and $g_n'$ be the step functions of the two samples
defining $\alpha_n$ and $\alpha_n'$.
Then
\[
    \frac{\|\alpha_n - \alpha_n'\|^2}{n}
    = \int_0^1 (g_n - g_n')^2
    \leq 2 \int_0^1 (g_n - G)^2 + 2 \int_0^1 (g_n' - G)^2
    \longrightarrow 0,
    \qquad
    \frac{\|\alpha_n\|^2}{n}
    = \int_0^1 g_n^2
    \longrightarrow m_2 > 0,
\]
almost surely, and therefore
$\|\beta_n - \rho\,\alpha_n\| / \|\alpha_n\|
= \rho\, \|\alpha_n' - \alpha_n\| / \|\alpha_n\| \to 0$.
\end{proof}

\section{Balanced bimodal spectra}\label{app:bimodal}

This appendix collects identities for the normalized spectra studied in
\S\ref{sec:experimentsQK} and \S\ref{sec:mechinterp}.
All matrix norms below are Frobenius norms,
$\|L\| = \sqrt{\operatorname{tr}(L^TL)}$.

Fix $m \geq 1$ and set $n = 2m$.
Let $\Lambda_m$ denote the set of unit vectors
$\lambda = (\lambda_1 \leq \cdots \leq \lambda_n) \in \RR^n$
with $\lambda_m < 0 < \lambda_{m+1}$.
For $\lambda \in \Lambda_m$ define
$\lambda_+, \lambda_- \in \RR^m$ by
\[
    (\lambda_+)_j = \lambda_{n+1-j},
    \qquad
    (\lambda_-)_j = -\lambda_j,
    \qquad
    j = 1,\ldots,m,
\]
so that $\lambda_+$ and $\lambda_-$ list the positive entries and the
absolute values of the negative entries of $\lambda$, both in
decreasing order. Define
\[
    \mathscr{I}(\lambda) = \|\lambda_+\|^2 - \|\lambda_-\|^2,
    \qquad
    C(\lambda) = \langle \lambda_+, \lambda_- \rangle,
    \qquad
    \mathscr{P}(\lambda) = 1 - \|\lambda_+ - \lambda_-\|.
\]
Finally, define the linear map $\iota \colon \RR^n \to \RR^n$ by
$(\iota\lambda)_i = -\lambda_{n+1-i}$.

\begin{lemma}\label{lem:bimodal-identities}
For $\lambda \in \Lambda_m$, abbreviating
$\mathscr{I}=\mathscr{I}(\lambda)$ and $C = C(\lambda)$:
\begin{enumerate}
    \item $\|\lambda_+\|^2 = \tfrac12(1+\mathscr{I})$ and
        $\|\lambda_-\|^2 = \tfrac12(1-\mathscr{I})$;
    \item $2\,\|\lambda_+\|\,\|\lambda_-\| = \sqrt{1-\mathscr{I}^2}$;
    \item $\|\lambda_+-\lambda_-\|^2 = 1 - 2C$, and hence
        $\mathscr{P}(\lambda) = 1 - \sqrt{1-2C}$;
    \item $\sum_{j=1}^m \bigl((\lambda_+)_j + i(\lambda_-)_j\bigr)^2
        = \mathscr{I} + 2iC$.
\end{enumerate}
\end{lemma}

\begin{proof}
Since the entries of $\lambda_+$ and $\lambda_-$ are the absolute
values of the entries of $\lambda$,
$\|\lambda_+\|^2 + \|\lambda_-\|^2 = \|\lambda\|^2 = 1$;
combining with the definition of $\mathscr{I}$ gives (1), and (2)
follows by multiplying the two expressions in (1).
For (3), expand
$\|\lambda_+-\lambda_-\|^2
= \|\lambda_+\|^2 + \|\lambda_-\|^2 - 2\langle\lambda_+,\lambda_-\rangle
= 1 - 2C$.
For (4), expand the square: the real part is
$\|\lambda_+\|^2 - \|\lambda_-\|^2 = \mathscr{I}$ and the imaginary
part is $2\langle\lambda_+,\lambda_-\rangle = 2C$.
\end{proof}

Note that $0 \leq 2C \leq 1$
($C \geq 0$ since the entries of $\lambda_\pm$ are nonnegative,
and $2C \leq 1$ by part (3)),
so $0 \leq \mathscr{P} \leq 1$.

\begin{lemma}\label{lem:bimodal-involution}
The map $\iota$ is an orthogonal, symmetric involution of $\RR^n$ and
maps $\Lambda_m$ to $\Lambda_m$, with
$\lambda_+(\iota\lambda) = \lambda_-(\lambda)$ and
$\lambda_-(\iota\lambda) = \lambda_+(\lambda)$.
Consequently
$\mathscr{I}\circ\iota = -\mathscr{I}$,
$C\circ\iota = C$, and
$\mathscr{P}\circ\iota = \mathscr{P}$.
Moreover, for $\lambda \in \Lambda_m$,
\[
    \|\lambda - \iota\lambda\|^2 = 2\,\|\lambda_+-\lambda_-\|^2,
    \qquad\text{so}\qquad
    \mathscr{P}(\lambda)
    = 1 - \tfrac{1}{\sqrt 2}\|\lambda - \iota\lambda\|;
\]
in particular $\mathscr{P}(\lambda) = 1$ if and only if
$\iota\lambda = \lambda$.
\end{lemma}

\begin{proof}
The matrix of $\iota$ has entries $-\delta_{j,\,n+1-i}$, which is
symmetric and orthogonal, and $\iota^2 = \id$.
If $\lambda$ is increasing then so is $\iota\lambda$, and the entries
of $\iota\lambda$ form the same multiset as the entries of $-\lambda$,
so $\iota(\Lambda_m) = \Lambda_m$.
The positive entries of $\iota\lambda$ are
$-\lambda_1 \geq \cdots \geq -\lambda_m$,
whence $\lambda_+(\iota\lambda) = \lambda_-(\lambda)$; the other
identity is symmetric, and the transformation rules for $\mathscr{I}$,
$C$, $\mathscr{P}$ follow.
For the displacement identity, the $j$-th and $(n+1-j)$-th coordinates
of $\lambda - \iota\lambda$ both equal
$\lambda_j + \lambda_{n+1-j} = (\lambda_+)_j - (\lambda_-)_j$
up to sign, so
$\|\lambda-\iota\lambda\|^2 = 2\sum_j
\bigl((\lambda_+)_j-(\lambda_-)_j\bigr)^2$.
The final statement follows from Lemma~\ref{lem:bimodal-identities}(3)
since $\mathscr{P}(\lambda)=1$ iff $\lambda_+=\lambda_-$.
\end{proof}

We prove Theorem~\ref{thm:simplex-faces}, using the notation of its
statement.
The argument is organized by the dimension of the face: the
partition of the energy first, then the facets, then the faces of
lower dimension, whose fibers are intersections of the facet
fibers.

\begin{proof}[Proof of Theorem~\ref{thm:simplex-faces}]
Rescaling, we may assume $\|L\| = 1$.

\emph{The partition.}
The entries of $\lambda_+$ and $\lambda_-$ are the absolute values
of the eigenvalues of $S$ together with the appended zeros, so
$\|\lambda_+\|^2 + \|\lambda_-\|^2 = \|S\|^2$;
expanding
$\|\lambda_+ - \lambda_-\|^2 = \|\lambda_+\|^2 + \|\lambda_-\|^2
- 2\langle\lambda_+,\lambda_-\rangle$
and using
$\|\lambda_+ - \lambda_-\|^2
= \|(\lambda_+ - \lambda_-)_+\|^2
+ \|(\lambda_+ - \lambda_-)_-\|^2$ gives
\[
    2\langle\lambda_+,\lambda_-\rangle
    + \|(\lambda_+ - \lambda_-)_+\|^2
    + \|(\lambda_+ - \lambda_-)_-\|^2
    = \|S\|^2.
\]
Adding $a = \|T\|^2$ and using
$\|S\|^2 + \|T\|^2 = \|L\|^2 = 1$
(Lemma~\ref{lemma:symmetric-skew-energy-baseline})
yields $a + b + c + d = 1$.
Nonnegativity is clear since the entries of $\lambda_+$ and
$\lambda_-$ are nonnegative.

\emph{The facets.}
We prove the vanishing conditions of (1)--(4); the vertex
conditions are proved below with the other zero-dimensional faces.
If $S = 0$ then $b = c = d = 0$, $S$ is semidefinite, and the
decompositions of (3) and (4) hold with $H = P = 0$; so assume
$S \neq 0$.

(1): $a = \|T\|^2$ vanishes if and only if $T = 0$.

(2): If $S$ is semidefinite then $\lambda_+ = 0$ or
$\lambda_- = 0$, so $b = 2\langle \lambda_+, \lambda_-
\rangle = 0$.
Conversely, if $S$ is indefinite then $(\lambda_+)_1 > 0$ and
$(\lambda_-)_1 > 0$; since the lists $\lambda_+$ and $\lambda_-$ are sorted in
decreasing order,
$\langle \lambda_+, \lambda_- \rangle
\geq (\lambda_+)_1 (\lambda_-)_1 > 0$, so $b > 0$.

(3): First, $d = \|(\lambda_+ - \lambda_-)_-\|^2$ vanishes iff
$(\lambda_+)_j \geq (\lambda_-)_j$ for every $j$, which
is the stated rank condition.

Assume the rank condition.
Choose orthonormal eigenvectors of $S$:
for each $j$ with $(\lambda_+)_j > 0$ let $v_j$ be an eigenvector
for the eigenvalue $(\lambda_+)_j$, and for each $j$ with
$(\lambda_-)_j > 0$ let $w_j$ be an eigenvector for
$-(\lambda_-)_j$, taking each eigenvalue with its multiplicity.
Since $(\lambda_-)_j > 0$ implies $(\lambda_+)_j > 0$, every
$w_j$ is matched with a $v_j$. Define
\[
    H = \sum_{j \,:\, (\lambda_-)_j > 0}
    (\lambda_-)_j \bigl(v_j v_j^T - w_j w_j^T\bigr),
    \qquad
    P = S - H = \sum_{j \,:\, (\lambda_+)_j > 0}
    \bigl((\lambda_+)_j - (\lambda_-)_j\bigr) \, v_j v_j^T.
\]
The spectrum of $H$ consists of the pairs
$\pm(\lambda_-)_j$ together with zeros, so it is symmetric
about zero; $P$ is positive semidefinite since
$(\lambda_+)_j \geq (\lambda_-)_j$; and
$\|P\|^2 = \sum_j \bigl((\lambda_+)_j - (\lambda_-)_j\bigr)^2
= \|(\lambda_+ - \lambda_-)_+\|^2 = c$.
All three matrices are diagonal in a common eigenbasis of $S$, so
they commute.

Conversely, assume $S = H + P$.
Let $\mu_1 \geq \cdots \geq \mu_N$ be the eigenvalues of $H$;
symmetry of the spectrum gives $\mu_{N+1-k} = -\mu_k$.
Writing $\lambda_k(S)$ for the $k$-th largest eigenvalue of $S$,
the Courant--Fischer min--max theorem
\cite[Theorem~4.2.6]{horn2013matrix} gives
\[
    \lambda_k(S)
    = \max_{\dim V = k}\, \min_{\substack{x \in V \\ \|x\| = 1}}
    x^T (H + P)\, x
    \;\geq\;
    \max_{\dim V = k}\, \min_{\substack{x \in V \\ \|x\| = 1}}
    x^T H x = \mu_k,
\]
since $x^T P x \geq 0$ for every $x$.
Fix $t > 0$. If $\lambda_k(S) < -t$ then $\mu_k < -t$, so
$\mu_{N+1-k} = -\mu_k > t$ and hence
$\lambda_{N+1-k}(S) \geq \mu_{N+1-k} > t$.
As $k \mapsto N+1-k$ is injective,
\begin{equation}\label{eqn:counting-dominance}
    \#\{k : \lambda_k(S) < -t\}
    \;\leq\;
    \#\{k : \lambda_k(S) > t\}
    \qquad \text{for every } t > 0.
\end{equation}
Now suppose the rank condition fails, say
$(\lambda_+)_j < (\lambda_-)_j$, and choose $t$ with
$(\lambda_+)_j < t < (\lambda_-)_j$.
Since $\lambda_+$ is nonincreasing, the eigenvalues of $S$
exceeding $t$ are among $(\lambda_+)_i$ for $i < j$, so the
right side of \eqref{eqn:counting-dominance} is at most $j - 1$;
since $(\lambda_-)_i \geq (\lambda_-)_j > t$ for
$i \leq j$, the left side is at least $j$. This contradicts
\eqref{eqn:counting-dominance}.

(4): Apply (3) to $-L$, which preserves $a$ and $b$, exchanges $c$
with $d$ and $\lambda_+$ with $\lambda_-$, and negates $H$ and $P$.

\emph{The vertices.}
Each vertex is the face on which the other three coordinates
vanish.
If $a = 1$ then $c = d = 0$ gives $\lambda_+ = \lambda_-$ by (3) and (4),
and $b = 0$ makes $S$ semidefinite by (2), so $\lambda_+ = \lambda_- = 0$
and $S = 0$: $L$ is antisymmetric.
Conversely, an antisymmetric $L$ has $S = 0$, so
$a = \|T\|^2 = 1$.
If $b = 1$ then $L$ is symmetric with $\lambda_+ = \lambda_-$ by (1), (3),
and (4); the converse is immediate from the definitions and the
partition.
If $c = 1$ then $L$ is symmetric by (1) and $S$ is semidefinite by
(2); were $S$ negative semidefinite, then $\lambda_+ = 0$ and
$d = \|\lambda_-\|^2 = \|S\|^2 = 1$, contradicting $d = 0$; so $L$ is
positive semidefinite.
Conversely, a symmetric positive semidefinite $L$ has $T = 0$ and
$\lambda_- = 0$, so $a = b = d = 0$ and $c = 1$.
The case $d = 1$ follows by applying this to $-L$, which
exchanges $\lambda_+$ with $\lambda_-$ and hence $c$ with $d$.
\end{proof}

\section{Moments}\label{app:moments}

Let $S \neq 0$ be the symmetric component of the bilinear form of a
head and let $\widehat S = S/\|S\|$,
where $\|S\| = \sqrt{\operatorname{tr}(S^2)}$ is the Frobenius norm.
Every eigenvalue of $\widehat S$ has absolute value less than one,
since $\widehat S$ has unit Frobenius norm and more than one nonzero
eigenvalue; hence $I - \widehat S^2$ is invertible. Define
\[
    O = \operatorname{tr}\bigl(\widehat S\,(I-\widehat S^2)^{-1}\bigr),
    \qquad
    E = \operatorname{tr}\bigl(\widehat S^2(I-\widehat S^2)^{-1}\bigr).
\]
In terms of the eigenvalues $\lambda_i$ of $\widehat S$,
\[
    O = \sum_i \frac{\lambda_i}{1-\lambda_i^2}
      = \sum_{k \text{ odd}} \sum_i \lambda_i^k,
    \qquad
    E = \sum_i \frac{\lambda_i^2}{1-\lambda_i^2}
      = \sum_{\substack{k \geq 2 \\ k \text{ even}}} \sum_i \lambda_i^k:
\]
$O$ is the sum of all odd moments of the spectrum and $E$ the sum of
all even moments
(Lemma~\ref{lem:OE-equivalences};
Lemma~\ref{lem:gram-form} expresses $O$ and $E$ directly in terms of
the key and query matrices).
Replacing $S$ by $-S$ changes the sign of $O$ and leaves $E$
unchanged.

We keep the notation of Appendix~\ref{app:bimodal}: $\Lambda_m$,
the involution $\iota$, and the sorted lists $\lambda_\pm$.

\begin{lemma}\label{lem:OE-equivalences}
Every $\lambda \in \Lambda_m$ satisfies $|\lambda_i| < 1$ for all $i$.
Define
\[
    O(\lambda) = \sum_i \frac{\lambda_i}{1-\lambda_i^2},
    \qquad
    E(\lambda) = \sum_i \frac{\lambda_i^2}{1-\lambda_i^2},
    \qquad
    R(\lambda) = \sum_i \frac{1}{1-\lambda_i}.
\]
Then:
\begin{enumerate}
    \item $O(\lambda)
        = \sum_{\substack{k \geq 1 \\ k \text{ odd}}} \sum_i \lambda_i^k$
        and
        $E(\lambda)
        = \sum_{\substack{k \geq 2 \\ k \text{ even}}} \sum_i \lambda_i^k$,
        both series converging absolutely;
    \item $O = \tfrac12(R - R\circ\iota)$ and
        $n + E = \tfrac12(R + R\circ\iota)$;
        in particular $O\circ\iota = -O$ and $E\circ\iota = E$;
    \item if $S$ is a symmetric matrix whose spectrum, with
        multiplicity, is the entries of $\lambda$, then
        \[
            O(\lambda) = \operatorname{tr}\bigl(S(I-S^2)^{-1}\bigr),
            \quad
            E(\lambda) = \operatorname{tr}\bigl(S^2(I-S^2)^{-1}\bigr),
            \quad
            R(\lambda) = \operatorname{tr}\bigl((I-S)^{-1}\bigr).
        \]
\end{enumerate}
\end{lemma}

\begin{proof}
If $|\lambda_i| = 1$ for some $i$ then, since $\|\lambda\| = 1$, all
other entries vanish, contradicting the presence of both a negative
and a positive entry; and $|\lambda_i| \leq \|\lambda\| = 1$. Hence
$|\lambda_i| < 1$.
For (1), apply
$\sum_{k \text{ odd}} t^k = t/(1-t^2)$ and
$\sum_{k \geq 2 \text{ even}} t^k = t^2/(1-t^2)$,
valid for $|t| < 1$, to each entry and sum over $i$.
For (2), the entries of $\iota\lambda$ form the same multiset as the
entries of $-\lambda$ (Lemma~\ref{lem:bimodal-involution}), so
$R(\iota\lambda) = \sum_i (1+\lambda_i)^{-1}$, and the claims follow
from
\[
    \frac12\left(\frac{1}{1-t}-\frac{1}{1+t}\right)
    = \frac{t}{1-t^2},
    \qquad
    \frac12\left(\frac{1}{1-t}+\frac{1}{1+t}\right)
    = 1 + \frac{t^2}{1-t^2}.
\]
The parity statements follow from (2) since
$\iota^2 = \id$.
Part (3) is the spectral mapping
$\operatorname{tr} f(S) = \sum_i f(\lambda_i)$
applied to $f(t) = t/(1-t^2)$, $t^2/(1-t^2)$, and $(1-t)^{-1}$.
\end{proof}

\begin{lemma}\label{lem:gram-form}
Let $K, Q \in M_{n\times N}(\RR)$,
let $S = \tfrac12(K^TQ + Q^TK)$, and let
\[
    M = \begin{pmatrix} K \\ Q \end{pmatrix}
    \colon \RR^N \to \RR^{2n},
    \qquad
    G = MM^T,
    \qquad
    J = \begin{pmatrix} 0 & I_n \\ I_n & 0 \end{pmatrix}.
\]
Then $S = \tfrac12 M^TJM$ and
\[
    \operatorname{tr}(S^k)
    = \operatorname{tr}\bigl((\tfrac12 JG)^k\bigr)
    \qquad (k \geq 1);
\]
in particular
$\|S\|^2 = \tfrac14\operatorname{tr}\bigl((JG)^2\bigr)$.
If moreover $M$ is surjective and $n \geq 1$,
then the nonzero spectrum of $S$, with multiplicity,
is the spectrum of $\tfrac12 JG$;
the sorted normalized nonzero spectrum $\lambda$ of $S$
lies in $\Lambda_n$, and with $H = JG/(2\|S\|)$,
\[
    O(\lambda) = \operatorname{tr}\bigl(H(I-H^2)^{-1}\bigr),
    \qquad
    E(\lambda) = \operatorname{tr}\bigl(H^2(I-H^2)^{-1}\bigr),
    \qquad
    R(\lambda) = \operatorname{tr}\bigl((I-H)^{-1}\bigr).
\]
\end{lemma}

\begin{proof}
Block multiplication gives
$M^TJM = K^TQ + Q^TK = 2S$.
By the cyclic property of the trace,
\[
    \operatorname{tr}\bigl((M^TJM)^k\bigr)
    = \operatorname{tr}\bigl(M^T(JG)^{k-1}JM\bigr)
    = \operatorname{tr}\bigl((JG)^{k-1}JMM^T\bigr)
    = \operatorname{tr}\bigl((JG)^k\bigr).
\]
For the second claim, recall that for any
$A \in M_{N\times 2n}(\RR)$ and $B \in M_{2n\times N}(\RR)$
the nonzero eigenvalues of $AB$ and $BA$ coincide with multiplicity;
applying this to $A = M^T$ and $B = \tfrac12 JM$ identifies the
nonzero spectrum of $S$ with the nonzero spectrum of
$\tfrac12 JMM^T = \tfrac12 JG$.
If $M$ is surjective then $G$ is positive definite.
The matrix $B = G^{1/2}JG^{1/2}$ is symmetric and congruent
to $J$, so by Sylvester's law of inertia it has exactly $n$ positive
and $n$ negative eigenvalues; and
$JG = G^{-1/2}B\,G^{1/2}$
is similar to $B$, hence diagonalizable with the same real spectrum.
Consequently the nonzero spectrum of $S$ consists of $n$ positive and
$n$ negative eigenvalues, so its sorted normalization $\lambda$ lies
in $\Lambda_n$ and, by Lemma~\ref{lem:OE-equivalences}, every entry
of $\lambda$ has absolute value less than one.
The matrix $H$ is diagonalizable with spectrum $\lambda$, and for a
diagonalizable matrix the trace of a rational function with poles off
the spectrum is the sum of its values on the spectrum, which gives
the three displayed identities.
\end{proof}

\section{Supplementary Figures}\label{app:score-figures}

This appendix collects supplementary figures for the experiments in
\S\ref{sec:experimentsQK}.

\begin{figure}[!htbp]
\centering
\includegraphics[width=\textwidth]{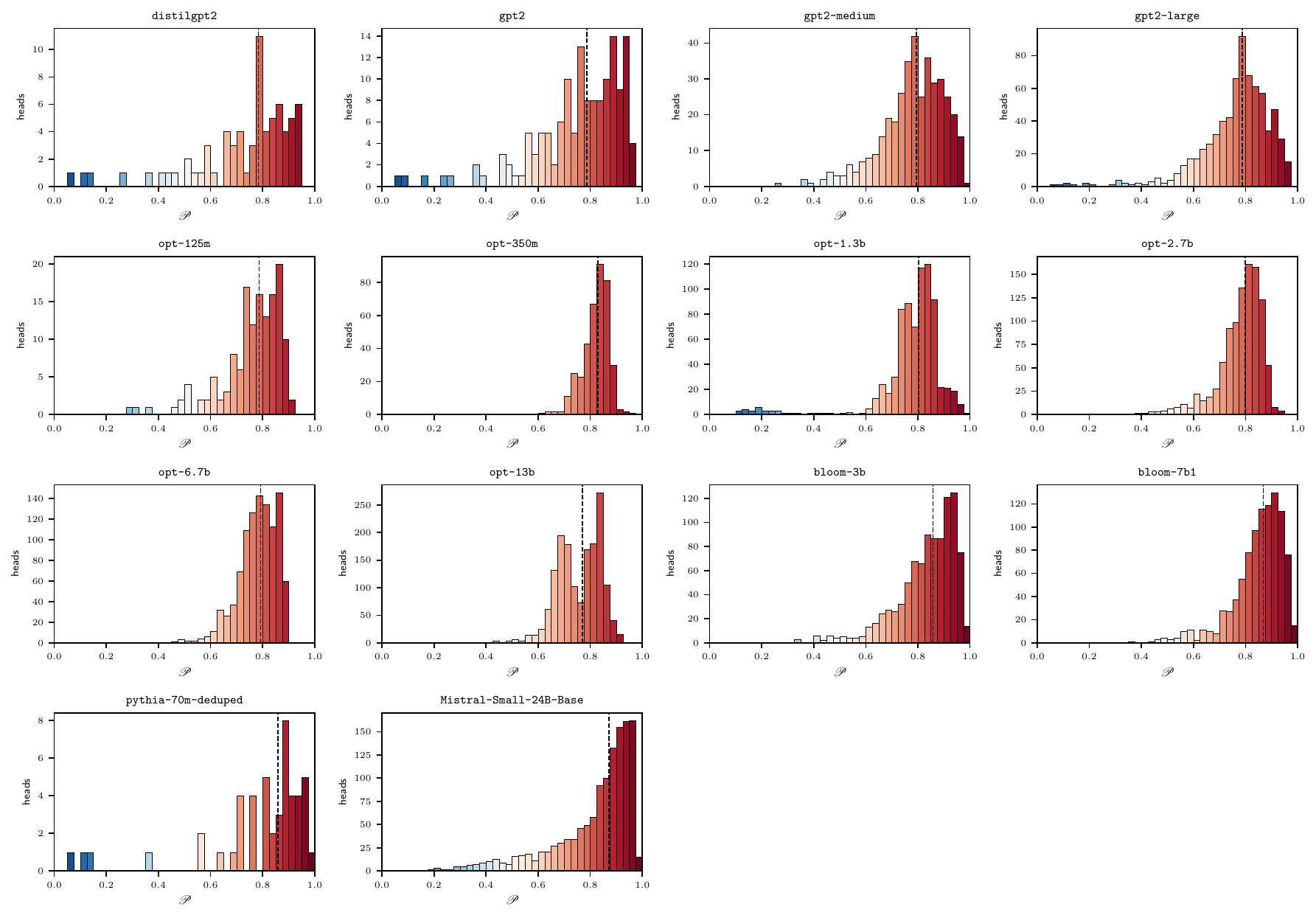}
\caption{Distribution of the trained-head pairing scores
$\mathscr{P}$ in Figure~\ref{fig:qk-random-baseline}.
Dashed lines mark model medians.}
\label{fig:pairing-histogram}
\end{figure}

\FloatBarrier
\clearpage
\begin{figure}[p]
\centering
\includegraphics[width=\textwidth,height=0.87\textheight,keepaspectratio]{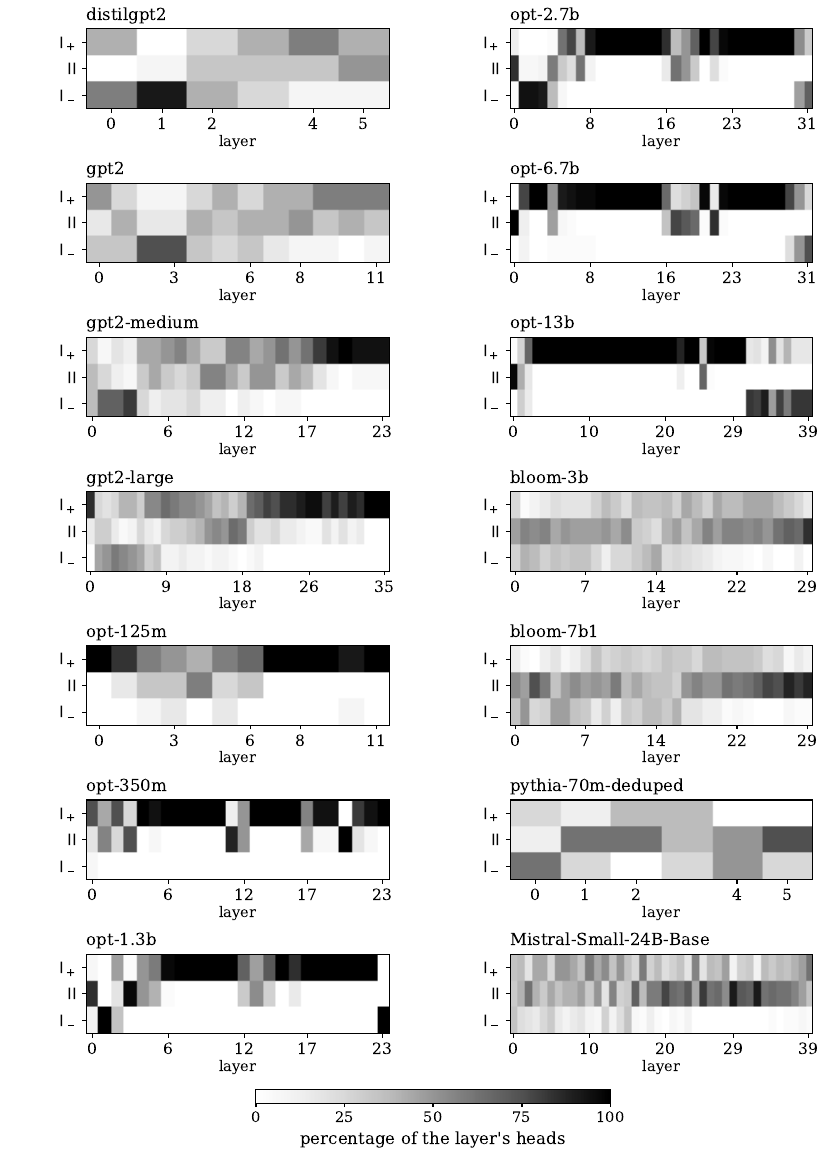}
\caption{Parity by layer in all fourteen models. Each cell shows the
percentage of a layer's heads of the indicated type, with a shared scale
from white ($0\%$) to black ($100\%$). Types use the symmetric profile:
Type~$\mathrm{II}$ has $c_S+d_S\leq1/50$; otherwise $c_S>d_S$ gives
Type~$\mathrm{I}_+$ and the remaining heads are Type~$\mathrm{I}_-$.
Layers are indexed from zero; RoPE models use relative position zero.}
\label{fig:parity-by-layer}
\end{figure}
\FloatBarrier

\begin{figure}[!htbp]
\centering
\includegraphics[width=0.9\textwidth,height=0.87\textheight,keepaspectratio]{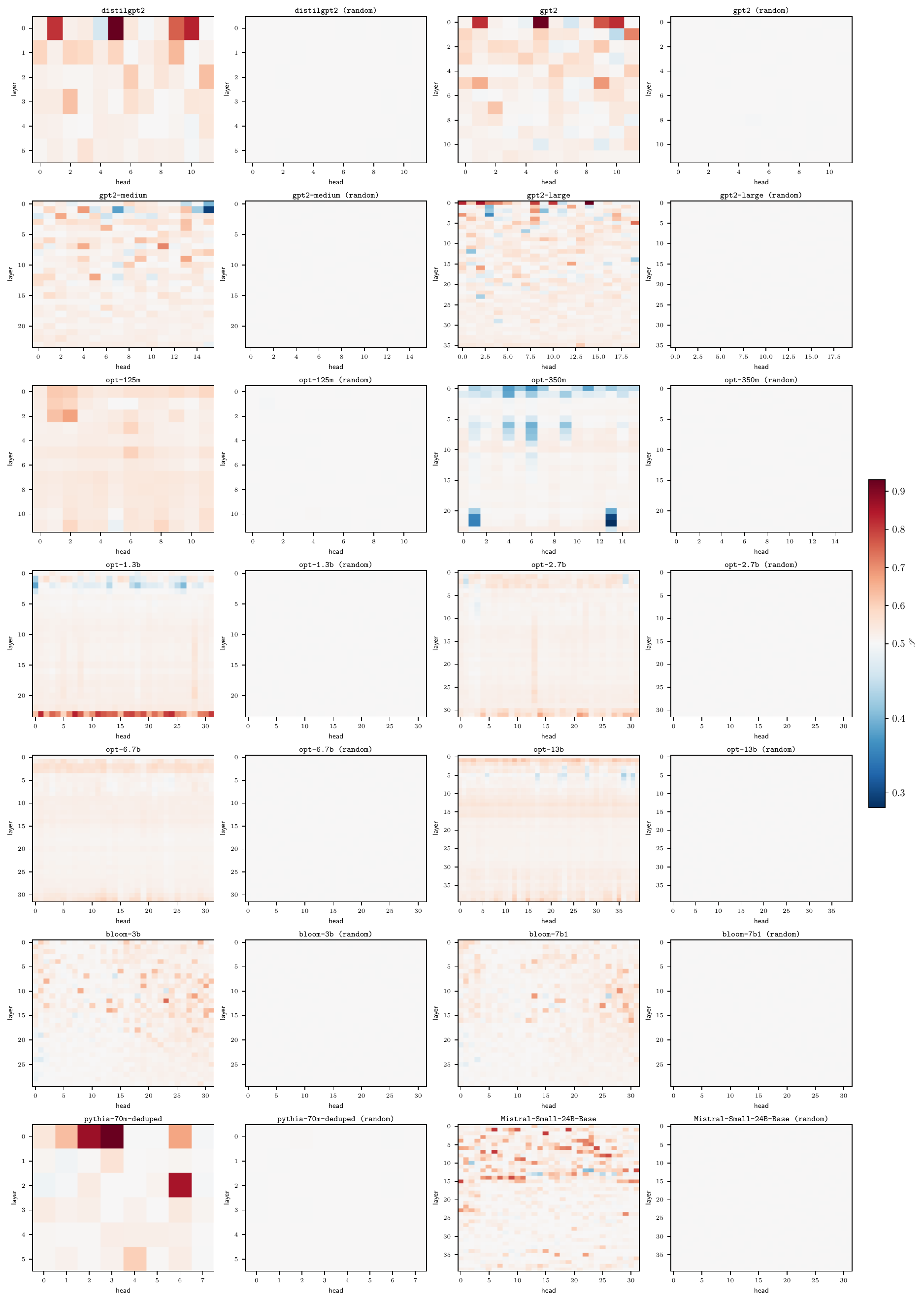}
\caption{Symmetric energy $\mathscr{S}=\|S\|^2/\|L\|^2$ at each
attention head. For each model, the trained heads (left) are paired
with independent random products $L=K^TQ$ (right), with iid
standard normal $K,Q\in\RR^{n\times N}$ of the same dimensions.
Rows index layers and columns index heads, both starting at zero.
The shared color scale is white at $1/2$, red above and blue below;
the random baseline has expectation $(N+1)/(2N)$.
RoPE models use relative position $d=0$.}
\label{fig:qk-multimodel}
\end{figure}

\begin{figure}[!htbp]
\centering
\includegraphics[width=\textwidth,height=0.87\textheight,keepaspectratio]{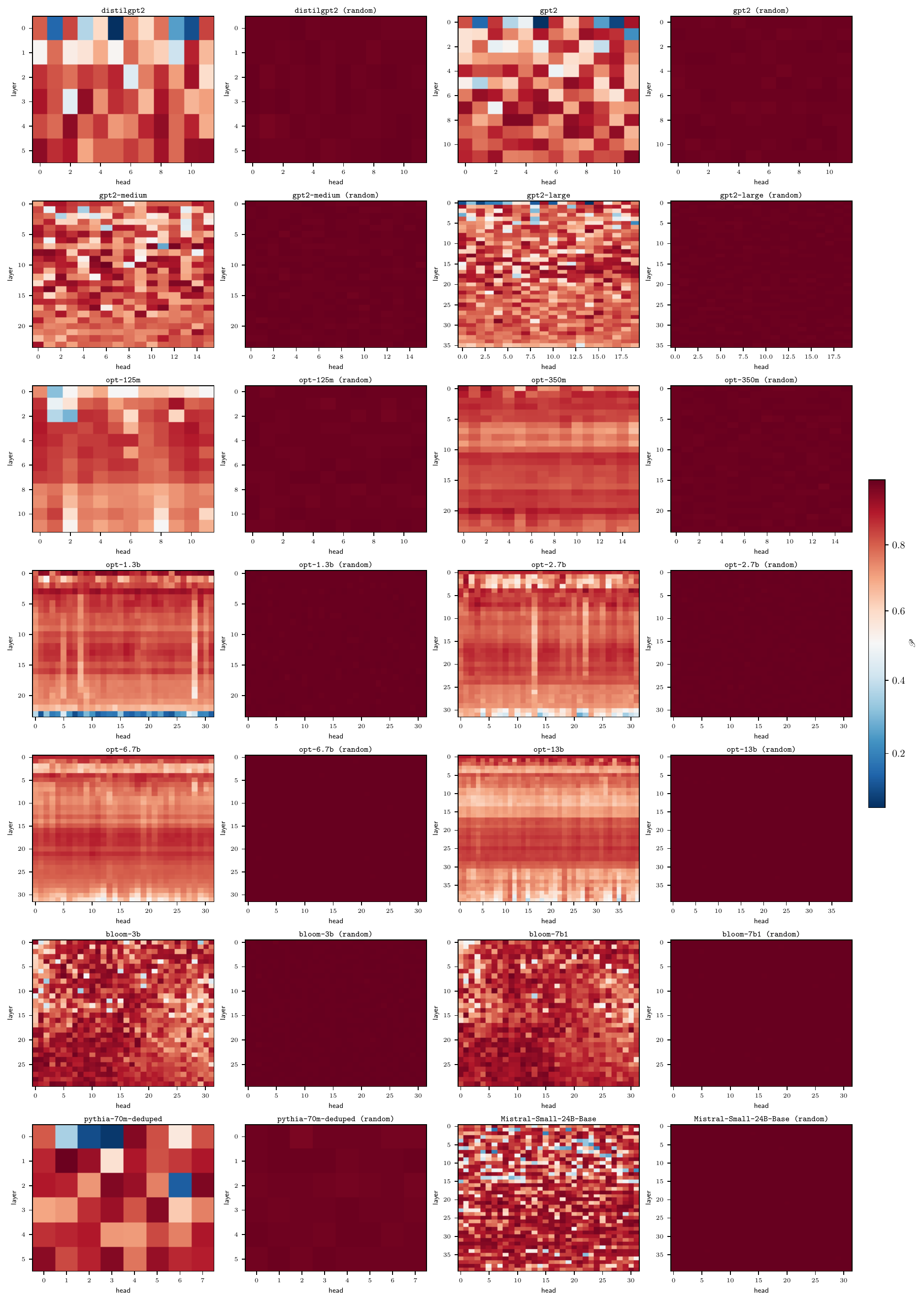}
\caption{Pairing score
$\mathscr{P}=1-\|\lambda_+-\lambda_-\|/\|S\|$ at each attention head.
Each model's trained heads (left) are paired with independent Gaussian
QK products of the same dimensions (right), as in
Figure~\ref{fig:qk-multimodel}. Rows index layers and columns index
heads. A score of one means the symmetric spectrum is symmetric about
zero. The shared color scale is white at $1/2$, red above and blue
below. RoPE models use relative position $d=0$.}
\label{fig:qk-random-baseline}
\end{figure}
\FloatBarrier

\bibliographystyle{abbrv}
\bibliography{draft}

\begin{thebibliography}{10}

\bibitem{biderman2023pythia}
S.~Biderman, H.~Schoelkopf, Q.~G. Anthony, H.~Bradley, K.~O'Brien, E.~Hallahan,
  M.~A. Khan, S.~Purohit, U.~S. Prashanth, E.~Raff, et~al.
\newblock Pythia: A suite for analyzing large language models across training
  and scaling.
\newblock In {\em International Conference on Machine Learning}, pages
  2397--2430. PMLR, 2023.

\bibitem{bigscience2022bloom}
{BigScience Workshop}, T.~L. Scao, A.~Fan, C.~Akiki, E.~Pavlick, S.~Ili{\'c},
  D.~Hesslow, R.~Castagn{\'e}, A.~S. Luccioni, F.~Yvon, M.~Gall{\'e}, J.~Tow,
  A.~M. Rush, S.~Biderman, et~al.
\newblock {BLOOM}: A 176b-parameter open-access multilingual language model.
\newblock {\em arXiv preprint arXiv:2211.05100}, 2022.

\bibitem{darveshi2025spectral}
S.~Darveshi.
\newblock Spectral taxonomy of {QK} circuits in transformer models.
\newblock LessWrong, 2025.
\newblock
  \url{https://www.lesswrong.com/posts/Yig9fc7wAxKqG63Do/spectral-taxonomy-of-qk-circuits-in-transformer-models}.

\bibitem{horn2013matrix}
R.~A. Horn and C.~R. Johnson.
\newblock {\em Matrix Analysis}.
\newblock Cambridge University Press, 2nd edition, 2013.

\bibitem{migliarini2025selfhating}
M.~Migliarini.
\newblock The self-hating attention head: A deep dive in {GPT-2}.
\newblock LessWrong / AI Alignment Forum, 2025.
\newblock
  \url{https://www.lesswrong.com/posts/wxPvdBwWeaneAsWRB/the-self-hating-attention-head-a-deep-dive-in-gpt-2-1}.

\bibitem{mistralai2025small3}
{Mistral AI Team}.
\newblock {Mistral Small 3}.
\newblock \url{https://mistral.ai/news/mistral-small-3/}, Jan. 2025.

\bibitem{press2022train}
O.~Press, N.~A. Smith, and M.~Lewis.
\newblock Train short, test long: Attention with linear biases enables input
  length extrapolation.
\newblock In {\em International Conference on Learning Representations}, 2022.

\bibitem{radford2019language}
A.~Radford, J.~Wu, R.~Child, D.~Luan, D.~Amodei, and I.~Sutskever.
\newblock Language models are unsupervised multitask learners.
\newblock {\em OpenAI Technical Report}, 2019.

\bibitem{sanh2019distilbert}
V.~Sanh, L.~Debut, J.~Chaumond, and T.~Wolf.
\newblock {DistilBERT}, a distilled version of {BERT}: smaller, faster, cheaper
  and lighter.
\newblock {\em arXiv preprint arXiv:1910.01108}, 2019.

\bibitem{saponati2025underlying}
M.~Saponati, P.~Sager, P.~V. Aceituno, T.~Stadelmann, and B.~Grewe.
\newblock The underlying structures of self-attention: symmetry,
  directionality, and emergent dynamics in transformer training.
\newblock {\em arXiv preprint arXiv:2502.10927}, 2025.

\bibitem{su2024roformer}
J.~Su, M.~Ahmed, Y.~Lu, S.~Pan, B.~Wen, and Y.~Liu.
\newblock {RoFormer}: Enhanced transformer with rotary position embedding.
\newblock {\em Neurocomputing}, 568:127063, 2024.

\bibitem{vaswani2017attention}
A.~Vaswani, N.~Shazeer, N.~Parmar, J.~Uszkoreit, L.~Jones, A.~N. Gomez,
  L.~Kaiser, and I.~Polosukhin.
\newblock Attention is all you need.
\newblock {\em arXiv preprint arXiv:1706.03762}, 2017.

\bibitem{zhang2022opt}
S.~Zhang, S.~Roller, N.~Goyal, M.~Artetxe, M.~Chen, S.~Chen, C.~Dewan, M.~Diab,
  X.~Li, X.~V. Lin, et~al.
\newblock {OPT}: Open pre-trained transformer language models.
\newblock {\em arXiv preprint arXiv:2205.01068}, 2022.

\end{thebibliography}

\end{document}